\documentclass[letterpaper]{article} 
\usepackage[preprint]{aaai2027}  
\usepackage[hyphens]{url}  
\usepackage{graphicx} 
\usepackage{natbib}  
\usepackage{caption} 
\usepackage{algorithm}
\usepackage{algorithmic}

\usepackage{threeparttable}
\usepackage{multirow}
\usepackage{booktabs}
\usepackage{makecell}
\usepackage[table]{xcolor}
\usepackage{pifont}
\usepackage{array}
\usepackage{xcolor}
\usepackage{amsmath}
\usepackage{amssymb}
\usepackage{amsthm}

\newtheorem{theorem}{Theorem}

\usepackage{newfloat}
\usepackage{listings}
\DeclareCaptionStyle{ruled}{labelfont=normalfont,labelsep=colon,strut=off} 
\floatstyle{ruled}
\newfloat{listing}{tb}{lst}{}
\floatname{listing}{Listing}

\usepackage{booktabs}

\title{MoNo: Multiscale Optimal Transport Neural Operator \\ for Solving PDEs on General Geometries}
\author{
    Zijiang Yang\textsuperscript{\rm 1},
    Xiaomeng Wu\textsuperscript{\rm 1},
    Dongmei Fu\textsuperscript{\rm 1}\textsuperscript{\rm 2}\corresponding \\
}
\affiliations{
    \textsuperscript{\rm 1}School of Automation and Electrical Engineering, University of Science and Technology Beijing\\
    \textsuperscript{\rm 2}Beijing Engineering Research Center of Industrial Spectrum Imaging\\

    zijiangyang@xs.ustb.edu.cn, fdm\_ustb@ustb.edu.cn
}

\begin{document}

\maketitle

\begin{abstract}
Transformer-based neural operators have achieved substantial progress in solving Partial Differential Equations (PDEs) by projecting spatial observations into compact latent tokens and learning physical interactions in latent spaces.
However, we reveal that existing learnable projection mechanisms cannot ensure stable and balanced assignments from observation points to latent tokens, causing some latent tokens to be over-assigned while others remain underutilized.
This limitation further restricts the design of hierarchical architectures, as assignment imbalance is continuously inherited and amplified across latent spaces, eventually causing severe token collapse in deeper spaces.
To address these issues, we propose \textbf{MoNo} (\textbf{M}ultiscale \textbf{O}ptimal Transport \textbf{N}eural \textbf{O}perator), a progressive multiscale neural operator that efficiently solves PDEs on general geometries through stable latent-space construction.
At its core is \textbf{CoTAP} (\textbf{C}ross-scale \textbf{O}ptimal \textbf{T}ransport \textbf{A}ssignment and \textbf{P}rojection), a novel latent-space construction method that formulates cross-space assignment between adjacent spaces as an entropy-regularized optimal transport problem, thereby constructing balanced bidirectional projections and stable latent spaces.
CoTAP also ensures stable information transfer across multiple latent spaces, further enabling multiscale architectures on general geometries, which in turn support more efficient learning of long-range physical interactions.
Extensive experiments demonstrate that MoNo outperforms existing state-of-the-art neural operators in both prediction performance and computational efficiency. Code is available at https://github.com/ZijiangY1116/MoNo.
\end{abstract}

\section{Introduction}

Partial Differential Equations (PDEs) are fundamental tools for describing the evolution of physical systems~\cite{raissi2020hidden, karniadakis2021physics,wang2025learning}.
Neural operators~\cite{fno}, particularly Transformer-based variants~\cite{kovachki2023neural,transolverpp}, have substantially advanced efficient PDE solving by learning operator mappings from input observations to target physical fields through sequence-based representations.
As standard self-attention scales quadratically with the number of observation points~\cite{vaswani2017attention}, vanilla Transformer-based neural operators are computationally prohibitive for large-scale PDE problems~\cite{transolver}, motivating extensive research on performing attention in compact latent spaces~\cite{lno,upt,linearno}.
However, the mechanisms for constructing latent spaces and the stability of cross-space mappings remain underexplored, limiting efficient compression of observation sequences and constraining flexibility in architectural design.
As illustrated in Figure~\ref{fig:first_fig}, SOTA methods not only exhibit limited prediction performance but also remain computationally expensive despite employing lower-complexity attention mechanisms~\cite{linearno}.

\begin{figure}[t]
    \centering
    \includegraphics[width=0.99\linewidth]{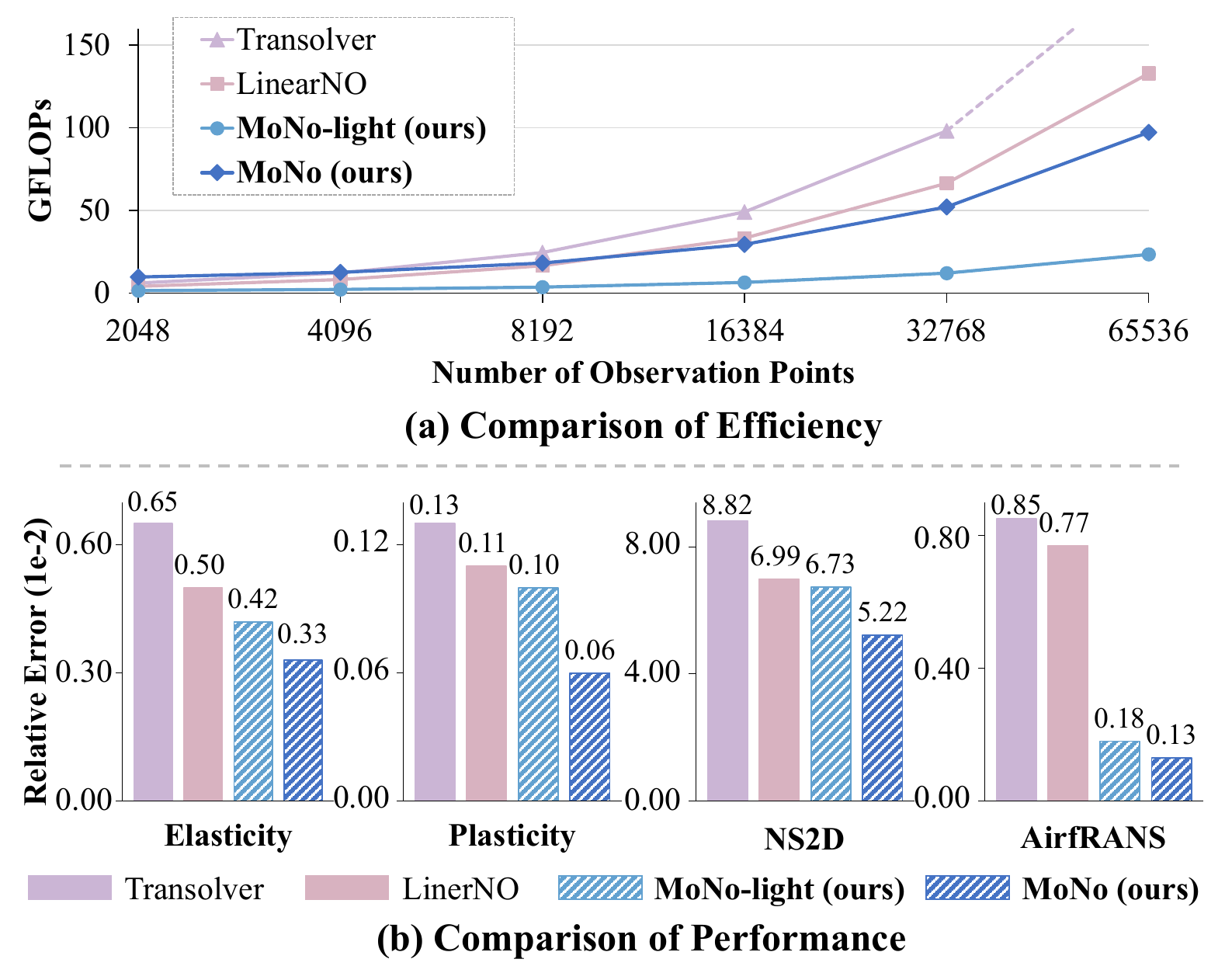}
    \caption{
    Comparison of MoNo with State-Of-The-Art (SOTA) methods.
    (a) MoNo achieves higher computational efficiency.
    (b) MoNo achieves lower prediction errors.
    }
    \label{fig:first_fig}
\end{figure}

To construct latent spaces, softmax-based projection is widely adopted, where the assignments from each observation point in the observation space to latent tokens in the latent space are represented by an unconstrained learnable projection matrix and normalized only afterward through independent row-wise or column-wise softmax operations~\cite{lno,linearno}.
In this work, we reveal that this latent-space construction lacks joint constraints on the assignment distribution, resulting in unstable and imbalanced assignments.
As shown in Figure~\ref{fig:motivation}(a), softmax-based projection concentrates dominant assignment weights on a small subset of latent tokens and produces substantially inconsistent assignment patterns between encoding and decoding.
Furthermore, when the model is extended into a multiscale latent-space architecture built with softmax-based projections, assignment imbalance is continuously inherited and amplified across latent spaces, eventually causing severe token collapse in deeper latent spaces (Figure~\ref{fig:motivation}(b)).
These issues reduce the effective capacity of latent representations, undermine stable cross-space information transfer, and hinder the design of multiscale architectures.

\begin{figure}[t]
    \centering
    \includegraphics[width=0.99\linewidth]{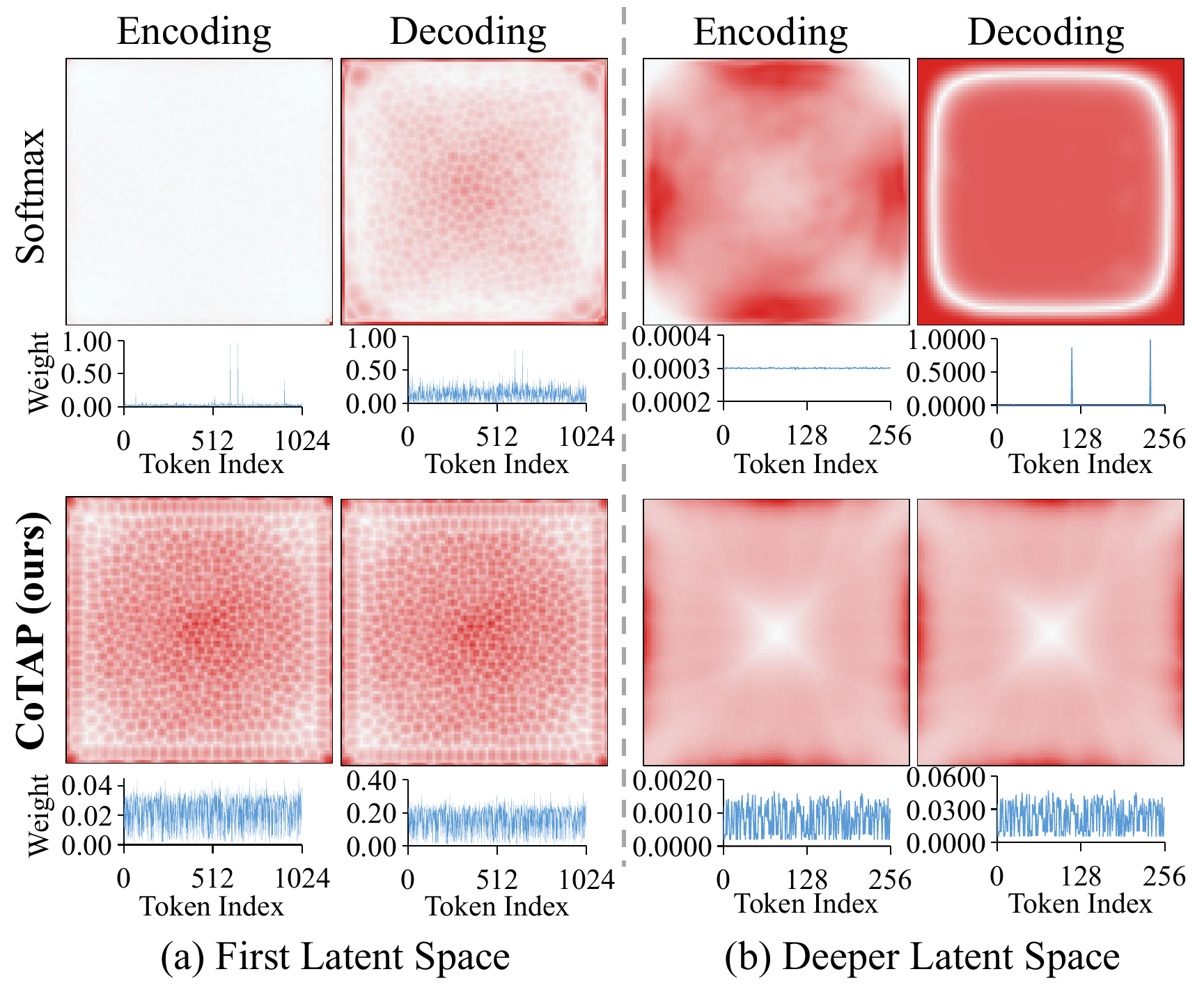}
    \caption{
    Comparison of assignment patterns on Darcy at a resolution of $85\times85$.
    The heatmaps visualize the maximum assignment of each observation point across latent tokens.
    In the heatmaps, colors transition from white to red as the assignment weights increase.
    The line plots show the maximum assignment weight of each latent token across observation points.
    Softmax-based projection not only fails to preserve consistency between encoding and decoding mappings, but also exhibits (a) imbalanced token assignments in the first latent space and (b) severe token collapse in the deeper latent space.
    }
    \label{fig:motivation}
\end{figure}

To address these issues, we propose \textbf{MoNo} (\textbf{M}ultiscale \textbf{O}ptimal Transport \textbf{N}eural \textbf{O}perator), an efficient framework for solving PDEs on general geometries by combining stable cross-space projections with a progressive multiscale hierarchy.
At its core is \textbf{CoTAP} (\textbf{C}ross-scale \textbf{O}ptimal \textbf{T}ransport \textbf{A}ssignment and \textbf{P}rojection), a novel latent-space construction method that formulates assignment learning between adjacent spaces as an entropy-regularized optimal transport problem, further constructing balanced bidirectional projections from a shared transport plan with uniform marginal constraints (Figure~\ref{fig:motivation}).
CoTAP further ensures stable information transfer across multiple latent spaces, thereby enabling efficient multiscale architectures on general geometries beyond standard grids~\cite{u-fno}.
Building on this hierarchy, MoNo progressively aggregates physical states into more compact latent spaces to model long-range physical interactions at substantially lower computational cost than repeatedly stacking single-scale blocks~\cite{transolver,linearno}, while fusing multiscale physical state tokens during decoding to predict the physical field.
Extensive experiments on seven widely used PDE benchmarks demonstrate that MoNo outperforms SOTA methods while maintaining superior computational efficiency.
Our contributions are summarized as follows:
\begin{itemize}
    \item We propose MoNo, a progressive multiscale neural operator for solving PDEs on general geometries by constructing a stable hierarchy of latent spaces at multiple compression scales to efficiently learn physical interactions across scales.
    
    \item We propose CoTAP, which formulates cross-space assignment as an entropy-regularized optimal transport problem and constructs balanced bidirectional projections, thereby enabling the stable construction of latent spaces.
    
    \item Extensive experiments on seven PDE benchmarks show that MoNo outperforms SOTA methods in both prediction performance and computational efficiency.
\end{itemize}

\section{Related Work}

\subsection{Learnable PDE Solver}

\noindent \textbf{Physics-Informed Neural Networks.}
Physics-Informed Neural Networks (PINNs) incorporate governing PDE constraints, initial conditions, and boundary conditions into the training objective to optimize a neural approximation of the target physical field~\cite{niaki2021physics, raissi2020hidden, wandel2022spline, yang2023dmis, luo2025physics, wu2026multi}.
However, PINNs generally require complete and explicit formulations of the governing PDEs and associated solution conditions, which are difficult to obtain in real-world problems.
In this work, we focus on observation-driven PDE solving without requiring complete prior definition of the physical system.

\noindent \textbf{Neural Operators.}
Neural operators have shown great potential for solving PDEs by learning operator mappings from input observations to target fields~\cite{lu2019deeponet,gno,fno, herde2024poseidon,mccabe2024multiple, zhou2024unisolver}.
Recently, Transformer-based neural operators represent discrete spatial observations and target physical fields as sequences, further formulating PDE solving as a sequence-to-sequence mapping~\cite{oformer}, which provides substantial flexibility in handling general geometries~\cite{transolver,transolverpp, ono}.
To efficiently process large-scale spatial observations, existing methods mainly introduce linear attention~\cite{oformer, factformer, gnot, linearno, transolverpp} to reduce the computational complexity and construct compact latent spaces to compress long observation sequences~\cite{lno, AROMA, upt, Aerogto, linearno, wen2026geometry}.

\begin{figure*}[t]
    \centering
    \includegraphics[width=0.95\linewidth]{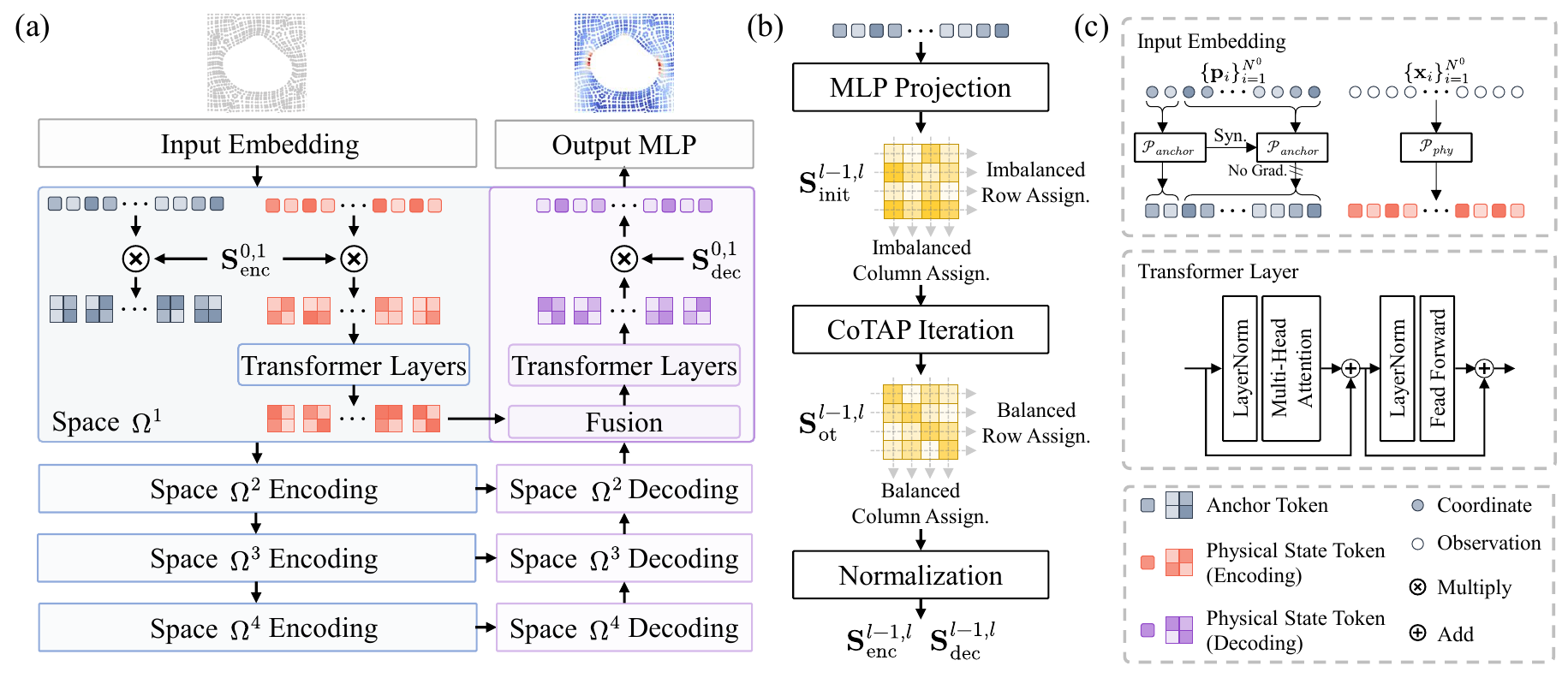}
    \caption{
    Illustration of MoNo.
    (a) Given a physical system, MoNo first embeds the spatial positions and input observations into high-dimensional tokens, and then efficiently learns physical interaction with a hierarchy of latent spaces.
    (b) CoTAP constructs an initial assignment matrix, iteratively solves for a balanced assignment matrix, and normalizes it to obtain the encoding matrix and decoding matrix.
    (c) Illustration of the input embedding and Transformer layer.
    }
    \label{fig:main_fig}
\end{figure*}

\subsection{Optimal Transport}

Optimal transport (OT) provides a unified mathematical framework for distribution matching across spaces~\cite{villani2009optimal}.
In neural operator research, OT has primarily been employed as a regularization constraint to improve transfer learning~\cite{yang2025physics} and dense reconstruction~\cite{ma2026physense}.
In addition, some studies use OT to align arbitrary input geometries with predefined grids~\cite{li2025geometric,qi2026spidersolver}.
In this work, we propose CoTAP, which directly predicts assignment relations between adjacent spaces and further constructs balanced bidirectional projections through OT, enabling stable cross-scale information transfer without predefined grids.

\section{Preliminaries}
We consider PDE-solving problems defined on a spatial domain $\Omega^0 \subset \mathbb{R}^{D_s}$, where $D_s$ denotes the spatial dimension.
Given input observations determined by the governing PDE, solution conditions, and external parameters, the goal is to estimate the corresponding target physical field.
Beyond settings restricted to regular grids, we focus on a general observation setting.
Let $P=\{\mathbf{p}_i\}_{i=1}^{N^0}$ denote a set of $N^0$ discrete spatial observation points in $\Omega^0$.
Given observations $X=\{\mathbf{x}_i\}_{i=1}^{N^0}$ at $P$, the model aims to predict the target physical field $Y=\{\mathbf{y}_i\}_{i=1}^{N^0}$ on the same point set, where $\mathbf{x}_i \in \mathbb{R}^{D_x}$ and $\mathbf{y}_i \in \mathbb{R}^{D_y}$, with $D_x$ and $D_y$ denoting the input observation and target physical field dimensions, respectively.
Each $\mathbf{x}_i$ is a tensor composed of task-specific physical quantities, such as material parameters.

\section{Method}

The overall framework of MoNo is illustrated in Figure~\ref{fig:main_fig}.
Given the observation point set $P$ and the corresponding input observations $X$, MoNo first embeds spatial positions and input observations into an anchor token set $G^0=\{\mathbf{g}_i^0\}_{i=1}^{N^0}$ and a physical state token set $F_{\mathrm{enc}}^0=\{\mathbf{f}_i^0\}_{i=1}^{N^0}$ by $\mathcal{P}_{\mathrm{anchor}}$ and $\mathcal{P}_{\mathrm{phy}}$, respectively, where $\mathbf{g}_i^0\in\mathbb{R}^{D_g}$ and $\mathbf{f}_i^0\in\mathbb{R}^{D_f}$ denote the anchor token and physical state token of the $i$-th observation point, and $D_g$ and $D_f$ denote dimensions.
As the mapping from $P$ to $G^0$ is smooth, gradients through $\mathcal{P}_{\mathrm{anchor}}$ are retained for only 25\% of randomly sampled $P$, while the remaining points are processed by a gradient-free copy, providing performance comparable to full-point backpropagation with lower training memory (Figure~\ref{fig:main_fig}(c)).
Furthermore, MoNo progressively builds stable latent spaces at multiple compressed scales to learn physical interactions.
Finally, the decoder maps multiscale latent tokens back to the original observation points and outputs the predicted target field $\hat{Y}=\{\hat{\mathbf{y}}_i\}_{i=1}^{N^0}$.

\subsection{CoTAP}
\noindent \textbf{Cross-space Assignment and Projection.}
For any $l$-th latent space $\Omega^l$, where $l\in\{1,\ldots,L\}$, we use CoTAP to construct the cross-space assignment and projection between the preceding space $\Omega^{l-1}$ and the current latent space $\Omega^l$.
$\Omega^0$ denotes the original observation space, while $\Omega^l$ for $l\geq1$ denotes the $l$-th latent space.
For the preceding space $\Omega^{l-1}$, we are given the anchor token set $G^{l-1}=\{\mathbf{g}_i^{l-1}\}_{i=1}^{N^{l-1}}$, where $N^{l-1}$ denotes the number of elements in $\Omega^{l-1}$ and $\mathbf{g}_i^{l-1}$ denotes the anchor token of the $i$-th element.

CoTAP first represents $G^{l-1}$ in matrix form $G^{l-1}\in\mathbb{R}^{N^{l-1}\times D_g}$ and constructs an initial assignment matrix $\mathbf{S}_{\mathrm{init}}^{l-1,l}$ through an MLP projector $\mathcal{P}^l(\cdot)$:
\begin{equation}
    \mathbf{S}_{\mathrm{init}}^{l-1,l}
    =
    \mathcal{P}^l(G^{l-1})
    \in \mathbb{R}^{N^{l-1}\times N^{l}},
\end{equation}
where $\mathbf{S}_{\mathrm{init}, ij}^{l-1,l}$ denotes the initial assignment score between the $i$-th element in $\Omega^{l-1}$ and the $j$-th latent token in $\Omega^{l}$.

CoTAP further treats $\mathbf{S}_{\mathrm{init}}^{l-1,l}$ as the transport reward and solves the entropy-regularized optimal transport problem:
\begin{equation}
    \mathbf{S}_{\mathrm{ot}}^{l-1,l}
    =
    \arg\max_{\mathbf{S}\in\mathcal{C}^{l-1,l}}
    \left\langle
    \mathbf{S},
    \mathbf{S}_{\mathrm{init}}^{l-1,l}
    \right\rangle
    +
    \tau\mathcal{H}(\mathbf{S}),
\label{eq:ot_problem}
\end{equation}
where $\mathbf{S}_{\mathrm{ot}}^{l-1,l}\in\mathbb{R}^{N^{l-1}\times N^l}$ denotes the OT-normalized assignment matrix, $\tau=1$ denotes the temperature coefficient, $\mathcal{H}(\mathbf{S})=-\sum_{i=1}^{N^{l-1}}\sum_{j=1}^{N^{l}}S_{ij}\log S_{ij}$ denotes the entropy regularization term, and $\mathcal{C}^{l-1,l}$ is the transport polytope with uniform marginal constraints.
These marginal constraints ensure that each input element sends the same total mass and each latent token receives the same total mass.

Based on $\mathbf{S}_{\mathrm{ot}}^{l-1,l}$, CoTAP constructs the bidirectional projection matrices between $\Omega^{l-1}$ and $\Omega^{l}$.
Since each column and row of $\mathbf{S}_{\mathrm{ot}}^{l-1,l}$ sums to $1/N^l$ and $1/N^{l-1}$, respectively, CoTAP further constructs projections through normalization:
\begin{equation}
\begin{aligned}
    \mathbf{S}_{\mathrm{enc}}^{l-1,l}
    &=
    N^l\mathbf{S}_{\mathrm{ot}}^{l-1,l},\\
    \mathbf{S}_{\mathrm{dec}}^{l-1,l}
    &=
    N^{l-1}\mathbf{S}_{\mathrm{ot}}^{l-1,l},
\end{aligned}
\end{equation}
where $\mathbf{S}_{\mathrm{enc}}^{l-1,l}$ denotes the projection from $\Omega^{l-1}$ to $\Omega^{l}$ and $\mathbf{S}_{\mathrm{dec}}^{l-1,l}$ denotes the projection from $\Omega^{l}$ to $\Omega^{l-1}$.
Since both projections are constructed from $\mathbf{S}_{\mathrm{ot}}^{l-1,l}$, CoTAP maintains a consistent mass-transfer relation in the encoding and decoding, rather than arbitrary learnable token mixing.

CoTAP solves Equation~\eqref{eq:ot_problem} with log-domain Sinkhorn iterations~\cite{sinkhorn1967diagonal}.
To reduce the computational cost for large-scale tasks, we further implement a Triton-based fused Sinkhorn solver, termed CoTAP Iteration.
Specifically, CoTAP Iteration directly fuses the log-sum-exp reductions into GPU kernels, thereby avoiding the materialization of large intermediate tensors.
For the backward pass, we derive and implement gradient computation equivalent to backpropagation through standard Sinkhorn iterations.
The details of CoTAP Iteration are provided in the Appendix.

We further demonstrate that the composition of CoTAP and a single latent self-attention layer can be represented as a learnable integral operator on the original physical domain $\Omega^0$, consistent with the operator interpretation of other methods~\cite{transolver, linearno}.
The detailed proof of Theorem~\ref{theorem:1} is provided in the Appendix.

\begin{theorem}[\textbf{CoTAP with latent self-attention as an integral operator on $\Omega^0$}]
\label{theorem:1}
Let $\Omega^0$ be the original physical domain endowed with a probability measure $\mu^0$.
For any bounded measurable physical state token field
$\mathcal{F}_{\mathrm{enc}}^0:\Omega^0\rightarrow\mathbb{R}^{D_f}$
and any continuous anchor token field
$\mathcal{G}^0:\Omega^0\rightarrow\mathbb{R}^{D_g}$,
the composition of CoTAP and a single latent self-attention layer can be represented as an input-dependent integral kernel operator on $\Omega^0$.
Specifically, for $\mu^0$-almost every $\mathbf{p}\in\Omega^0$,
\begin{equation}
\mathcal{T}
\bigl(\mathcal{G}^0,\mathcal{F}_{\mathrm{enc}}^0\bigr)(\mathbf{p})
=
\int_{\Omega^0}
\kappa
\bigl(\mathbf{p},\boldsymbol{\xi}\bigr)
\mathcal{F}_{\mathrm{enc}}^0(\boldsymbol{\xi})
\mathbf{W}_v
\,d\mu^0(\boldsymbol{\xi}),
\end{equation}
where $\mathcal{T}$ denotes the composite operator, $\mathbf{W}_v$ denotes the value projection in latent self-attention, and $\kappa$ denotes the input-dependent kernel.
\end{theorem}

\subsection{Progressive Multiscale Modeling}
Based on CoTAP, MoNo builds an encoding-decoding hierarchy with $L$ latent spaces to learn physical interactions at progressively compressed scales.

\noindent \textbf{Encoding.}
For each latent space $\Omega^l$, where $l\in\{1,\ldots,L\}$, the encoding stage projects anchor tokens and physical state tokens from $\Omega^{l-1}$ to $\Omega^l$ and learns physical interactions:
\begin{equation}
\begin{aligned}
    F_{\mathrm{enc}}^l
    &=
    \mathcal{M}_{\mathrm{enc}}^l
    \left(
    \left(\mathbf{S}_{\mathrm{enc}}^{l-1,l}\right)^{\top}
    F_{\mathrm{enc}}^{l-1}
    \right),\\
    G^l
    &=
    \left(\mathbf{S}_{\mathrm{enc}}^{l-1,l}\right)^{\top}
    G^{l-1},
\end{aligned}
\label{eq:encoding}
\end{equation}
where $\mathcal{M}_{\mathrm{enc}}^l$, $F_{\mathrm{enc}}^l$, and $G^l$ denote the encoding transformer block, the updated physical state token set, and the projected anchor token set in $\Omega^l$, respectively.
By repeating Equation~\eqref{eq:encoding} progressively, MoNo preserves local geometry and fine-grained physical information in shallow latent spaces while forming more compressed and global physical state representations in deeper latent spaces.

\begin{table*}[t]
  \centering
  \small
  \renewcommand\tabcolsep{15pt}
  \renewcommand{\arraystretch}{1.05}
  \begin{tabular}{l|ccccc}
    \toprule
    \multirow{2}{*}{Method} & \multicolumn{3}{c}{Structured Mesh} & \multicolumn{1}{c}{Standard Grid} & \multicolumn{1}{c}{Point Cloud} \\
    \cmidrule(lr){2-4} \cmidrule(lr){5-5} \cmidrule(lr){6-6}
    & Airfoil & Pipe & Plasticity & Navier-Stokes & Elasticity \\
    \midrule
    FNO~\cite{fno} & - & - & - & 0.1556 & - \\
    GEO-FNO~\cite{geo-fno} & 0.0138 & 0.0067 & 0.0074 & 0.1556 & 0.0229 \\
    F-FNO~\cite{f-fno} & 0.0078 & 0.0070 & 0.0047 & 0.2322 & 0.0263 \\
    Galerkin~\cite{galerkin} & 0.0118 & 0.0098 & 0.0120 & 0.1401 & 0.0240 \\
    OFormer (Li et al. 2023) & 0.0183 & 0.0168 & 0.0017 & 0.1705 & 0.0183 \\
    GNOT~\cite{gnot} & 0.0076 & 0.0047 & 0.0336 & 0.1380 & 0.0086 \\
    FactFormer (Li et al. 2023) & 0.0071 & 0.0060 & 0.0312 & 0.1214 & - \\
    ONO~\cite{ono} & 0.0061 & 0.0052 & 0.0048 & 0.1195 & 0.0118 \\
    LSM~\cite{lsm} & 0.0059 & 0.0050 & 0.0025 & 0.1535 & 0.0218 \\
    LNO~\cite{lno} & 0.0053 & 0.0031 & 0.0028 & 0.0830 & 0.0066 \\
    Transolver~\cite{transolver} & 0.0053 & 0.0030 & 0.0013 & 0.0882 & 0.0065 \\
    Transolver++~\cite{transolverpp} & 0.0051 & 0.0027 & 0.0014 & 0.1010 & 0.0064 \\
    LinearNO~\cite{linearno} & \underline{0.0049} & \underline{0.0024} & 0.0011 & 0.0699 & 0.0050 \\
    \midrule
    \textbf{MoNo-Light (ours)} & \textbf{0.0048} & 0.0027 & \underline{0.0010} & \underline{0.0673} & \underline{0.0042}\\
    \textbf{MoNo (ours)} & \textbf{0.0048} & \textbf{0.0021} & \textbf{0.0006} & \textbf{0.0522} & \textbf{0.0033} \\
    \bottomrule
  \end{tabular}
  \caption{Comparison on standard benchmarks in relative L2 ($\downarrow$). "-" denotes that the method is not applicable to the corresponding task. The best results are highlighted in bold, and the second-best results are underlined. MoNo outperforms other methods.}
  \label{tab:standard_bench}
\end{table*}

\noindent \textbf{Decoding.}
Starting from the deepest latent space $\Omega^L$, the decoding stage projects physical state tokens back through the hierarchy and fuses them with the encoded tokens:
\begin{equation}
    F_{\mathrm{dec}}^l
    =
    \begin{cases}
    \mathcal{M}_{\mathrm{dec}}^l(F_{\mathrm{enc}}^l),
    & l=L,\\
    \mathcal{M}_{\mathrm{dec}}^l
    \left(
    F_{\mathrm{enc}}^l
    +
    \mathbf{S}_{\mathrm{dec}}^{l,l+1}
    F_{\mathrm{dec}}^{l+1}
    \right),
    & 1\leq l<L,
    \end{cases}
\label{eq:decoding}
\end{equation}
where $\mathcal{M}_{\mathrm{dec}}^l$ and $F_{\mathrm{dec}}^l$ denote the decoding transformer block and the decoded physical state token set in $\Omega^l$, respectively.
Finally, $F_{\mathrm{dec}}^1$ is projected back to the original observation space as $F_{\mathrm{dec}}^0=\mathbf{S}_{\mathrm{dec}}^{0,1}F_{\mathrm{dec}}^1$.
An MLP output head then maps $F_{\mathrm{dec}}^0$ to the target physical field prediction $\hat{Y}$.

Furthermore, we demonstrate that MoNo defines a neural operator that maps input functions to target physical-field functions.
The detailed proof of Theorem~\ref{theorem:2} is provided in the Appendix.

\begin{theorem}[\textbf{MoNo as a neural operator}]
\label{theorem:2}
Given a continuous anchor token field
$\mathcal{G}^0:\Omega^0\rightarrow\mathbb{R}^{D_g}$
and a bounded measurable physical state token field
$\mathcal{F}_{\mathrm{enc}}^0:\Omega^0\rightarrow\mathbb{R}^{D_f}$,
MoNo with $L$ latent spaces defines the following learnable function-to-function mapping:
\begin{equation}
\widehat{\mathcal{Y}}
=
\mathcal{N}_{\theta}
\bigl(
\mathcal{G}^0,
\mathcal{F}_{\mathrm{enc}}^0
\bigr),
\qquad
\widehat{\mathcal{Y}}:
\Omega^0\rightarrow\mathbb{R}^{D_y},
\end{equation}
where $\mathcal{N}_{\theta}$ denotes the MoNo neural operator parameterized by the learnable parameters $\theta$, and $\widehat{\mathcal{Y}}$ denotes the predicted physical-field function defined on $\Omega^0$.
\end{theorem}

\subsection{MoNo Variants}

We define two variants and adjust their input and output dimensions according to each task.
Both variants contain four latent spaces, with the numbers of encoding and decoding Transformer layers set to $\{3,1,1,1\}$, and differ as follows:
\begin{itemize}
    \item \textbf{MoNo-light}: $D_g=96$, $D_f=96$, with latent token numbers of $\{512,256,128,64\}$.
    \item \textbf{MoNo}: $D_g=192$, $D_f=192$, with latent token numbers of $\{1024,512,256,128\}$.
\end{itemize}
MoNo-light has a parameter count comparable to existing methods, whereas MoNo uses a larger configuration to increase model capacity.
Notably, as shown in the Experiments section, both variants achieve higher computational efficiency than existing SOTA methods on large-scale tasks.

\section{Experiments}

\subsection{Experiment Settings}

\noindent \textbf{Benchmarks.}
To comprehensively evaluate MoNo across diverse geometries and physical systems, we conduct experiments on six widely used standard benchmarks~\cite{fno,linearno,lno}, including Airfoil, Pipe, Plasticity, Navier-Stokes (NS2D), Elasticity, and Darcy.
Following recent studies~\cite{transolver,linearno}, we further evaluate MoNo on AirfRANS~\cite{airfrans}, a real-world advanced benchmark.

\begin{figure*}[t]
    \centering
    \includegraphics[width=0.95\linewidth]{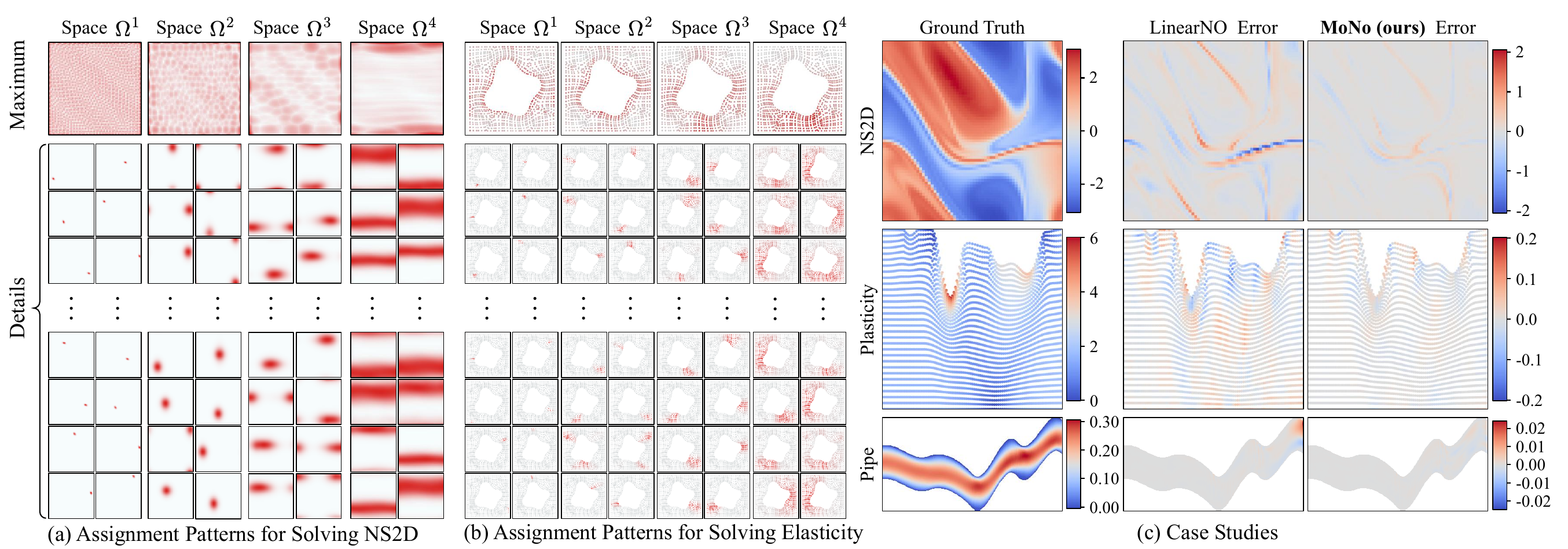}
    \caption{
    Visualization of assignment patterns and prediction errors.
    (a) The assignment patterns learned by MoNo on NS2D. (b) The assignment patterns learned by MoNo-light on Elasticity.
    (c) Case studies.
    MoNo constructs assignment relations with clear multiscale characteristics and achieves substantially lower prediction errors.
    }
    \label{fig:case_study}
\end{figure*}

\noindent \textbf{Baselines.}
We compare MoNo with SOTA neural operator methods, including FNO~\cite{fno}, GEO-FNO~\cite{geo-fno}, F-FNO~\cite{f-fno}, Galerkin Transformer~\cite{galerkin}, OFormer~\cite{oformer}, GNOT~\cite{gnot}, FactFormer~\cite{factformer}, ONO~\cite{ono}, LSM~\cite{lsm}, LNO~\cite{lno}, Transolver~\cite{transolver}, Transolver++~\cite{transolverpp}, and LinearNO~\cite{linearno}.
For AirfRANS, we additionally include MLP, PointNet~\cite{pointnet}, Graph U-Net~\cite{graphunet}, and MeshGraphNet~\cite{meshgraphnet}.

\begin{table}[!t]
  \centering
  \small
  \renewcommand\tabcolsep{2pt}
  \renewcommand{\arraystretch}{1.05}
  \begin{tabular}{l|cccc}
    \toprule
    Method & Vol. ($\downarrow$) & Surf. ($\downarrow$) & $C_L$ ($\downarrow$) & $\rho_L$ ($\uparrow$) \\
    \midrule
    MLP & 0.0081 & 0.0200 & 0.2108 & 0.9932 \\
    PointNet~\cite{pointnet} & 0.0253 & 0.0996 & 0.1973 & 0.9919 \\
    Graph U-Net (2017) & 0.0076 & 0.0144 & 0.1677 & 0.9949 \\
    MeshGraphNet (2020) & 0.0214 & 0.0387 & 0.2252 & 0.9945 \\
    \midrule
    GNO~\cite{gno} & 0.0269 & 0.0405 & 0.2016 & 0.9938 \\
    GEO-FNO~\cite{geo-fno} & 0.0361 & 0.0301 & 0.6161 & 0.9257 \\
    GALERKIN~\cite{galerkin} & 0.0074 & 0.0159 & 0.2336 & 0.9951 \\
    GNOT~\cite{gnot} & 0.0049 & 0.0152 & 0.1992 & 0.9942 \\
    GINO~\cite{gino} & 0.0297 & 0.0482 & 0.1821 & 0.9958 \\
    LNO~\cite{lno} & 0.0214 & 0.0268 & 0.1480 & 0.9744 \\
    Transolver~\cite{transolver} & 0.0023 & 0.0085 & 0.1230 & 0.9978 \\
    Transolver++ (2025) & 0.0068 & 0.0159 & 0.1880 & 0.9910 \\
    LinearNO~\cite{linearno} & \underline{0.0011} & 0.0077 & \underline{0.0491} & \textbf{0.9992} \\
    \midrule
    \textbf{MoNo-Light (ours)} & 0.0025 & \underline{0.0018} & 0.0616 & \underline{0.9986} \\
    \textbf{MoNo (ours)} & \textbf{0.0009} & \textbf{0.0013} & \textbf{0.0415} & \textbf{0.9992} \\
    \bottomrule
  \end{tabular}
  \caption{Comparison on the AirfRANS benchmark. Following previous works~\cite{transolver, linearno}, we evaluate the prediction error of surrounding (Vol.) and surface (Surf.) physical fields, as well as the lift coefficient ($C_L$). $\rho_L$ denotes Spearman's rank correlation for $C_L$. MoNo outperforms other methods.
  }
  \label{tab:airfrans_bench}
\end{table}

\noindent \textbf{Implementations.}
Please refer to the Appendix for detailed hyper-parameters and training configurations.

\subsection{Main Results}

\noindent \textbf{Standard Benchmarks.}
Table~\ref{tab:standard_bench} reports the results on standard benchmarks.
MoNo achieves the best performance across all tasks.
For the three structured-mesh tasks, including Airfoil, Pipe, and Plasticity, MoNo outperforms existing methods by at least 2.0\%, 12.5\%, and 45.5\% in terms of relative L2 error, respectively.
MoNo also achieves the best performance for complex time-dependent physical systems and point-cloud representations.

For fair comparisons, we also report the performance of MoNo-light, whose parameter count is comparable to that of SOTA methods.
MoNo-light still shows competitive and overall superior solving accuracy.
Specifically, MoNo-light outperforms SOTA methods by at least 2.0\%, 9.1\%, 3.7\%, and 16.0\% on Airfoil, Plasticity, NS2D, and Elasticity in terms of relative L2 error, respectively.
These results demonstrate that the stable cross-space projections constructed by CoTAP and the progressive multiscale hierarchy enable effective learning of multiscale physical states across diverse geometries and physical systems.

\noindent \textbf{Advanced Benchmarks.}
Table~\ref{tab:airfrans_bench} reports the results on the AirfRANS advanced benchmark.
MoNo achieves the best performance in predicting the surrounding physical fields, surface physical fields, and lift coefficient.
Specifically, MoNo outperforms the existing SOTA methods by at least 18.2\% and 83.1\% in predicting the surrounding physical fields and surface physical fields, respectively.
For aerodynamic performance evaluation, MoNo achieves an error of 0.0415 in predicting the lift coefficient, outperforming all existing methods by at least 15.5\%.
MoNo-light also outperforms all baselines by at least 76.6\% in predicting the surface physical fields.
Moreover, MoNo-light surpasses most existing SOTA methods on the remaining metrics.
These results further demonstrate that MoNo can handle complex geometries and real-world problems.

\begin{table}[!t]
  \centering
  \small
  \renewcommand\tabcolsep{8pt}
  \renewcommand{\arraystretch}{1.05}
  \begin{tabular}{l|ccccc}
    \toprule
    Method & 85$\times$85 & 141$\times$141 & 211$\times$211\\
    \midrule
    LNO & 0.0081 & 0.0074 & 0.0073\\
    Transolver & \underline{0.0055} & 0.0064 & 0.0063\\
    LinearNO & \textbf{0.0050} & \underline{0.0053} & \underline{0.0053}\\
    \midrule
    \textbf{MoNo-Light (ours)} & 0.0070 & 0.0057 & 0.0059\\
    \textbf{MoNo (ours)} & 0.0059 & \textbf{0.0049} & \textbf{0.0046}\\
    \bottomrule
  \end{tabular}
    \caption{Comparison on Darcy at multiple resolutions in relative L2 error ($\downarrow$). MoNo effectively reduces prediction errors as the spatial resolution increases.}
  \label{tab:multi_resolution_eval}
\end{table}

\begin{table}[t]
  \centering
  \small
  \setlength{\tabcolsep}{3pt}
  \renewcommand{\arraystretch}{1.05}
  \begin{tabular}{l|cc|cc}
    \toprule
    \multirow{2}{*}{Method}
    & \multicolumn{2}{c|}{OOD Reynolds}
    & \multicolumn{2}{c}{OOD Angles} \\
    \cmidrule(lr){2-3} \cmidrule(lr){4-5}
    & Vol. ($\downarrow$) & Surf. ($\downarrow$)
    & Vol. ($\downarrow$) & Surf. ($\downarrow$) \\
    \midrule
    MLP
    & 0.0669 & 0.1153
    & 0.1309 & 0.3311 \\

    PointNet
    & 0.0838 & 0.1403
    & 0.2021 & 0.4649 \\

    Graph U-Net
    & 0.0538 & 0.1168
    & 0.0979 & 0.2391 \\

    MeshGraphNet
    & 0.2789 & 0.2382
    & 0.4902 & 1.1071 \\

    \midrule
    GNO
    & 0.0833 & 0.1562
    & 0.1626 & 0.2359 \\

    Galerkin
    & 0.0330 & 0.0972
    & 0.0577 & 0.2773 \\

    GNOT
    & 0.0305 & 0.0959
    & 0.0471 & 0.3466 \\

    GINO
    & 0.0839 & 0.1825
    & 0.1589 & 0.2469 \\

    LNO
    & 0.0825 & 0.1762
    & 0.0346 & 0.0790 \\

    Transolver
    & 0.0122 & 0.0550
    & 0.0480 & 0.2335 \\

    LinearNO
    & 0.0112 & 0.0372
    & 0.0464 & 0.2500 \\

    \midrule
    \textbf{MoNo-light (ours)}
    & \underline{0.0090} & \underline{0.0146}
    & \underline{0.0222} & \underline{0.0553} \\

    \textbf{MoNo (ours)}
    & \textbf{0.0066} & \textbf{0.0121}
    & \textbf{0.0097} & \textbf{0.0248} \\

    \bottomrule
  \end{tabular}
  \caption{Comparison on the AirfRANS OOD benchmarks.
  }
  \label{tab:ood_benchmarks}
\end{table}

\noindent \textbf{Evaluation at Multiple Resolutions.}
To further evaluate the ability of models to utilize spatial observations of varying densities, we conduct experiments on Darcy at three spatial resolutions: $85\times85$, $141\times141$, and $211\times211$.
The results are reported in Table~\ref{tab:multi_resolution_eval}.
Both MoNo and MoNo-light effectively utilize the additional information provided by higher-resolution observations, achieving more accurate physical field predictions overall as the number of observation points increases.
In contrast, both Transolver and LinearNO exhibit performance degradation at higher resolutions.
In particular, at the resolution of $211\times211$, MoNo outperforms all other methods by at least 13.2\% in terms of relative L2 error.
These results further demonstrate that MoNo can effectively aggregate and propagate physical information across different observation resolutions, thereby improving solving performance as the observation density increases.

\noindent \textbf{Out-of-Distribution Evaluation.}
We further introduce out-of-distribution (OOD) experiments on the AirfRANS benchmark.
Specifically, following previous work~\cite{transolver}, we consider two OOD settings on AirfRANS: Reynolds-number extrapolation (OOD Reynolds) and angle-of-attack extrapolation (OOD Angles).
As shown in Table~\ref{tab:ood_benchmarks}, MoNo and MoNo-light outperform existing methods.
MoNo outperforms existing methods by at least 41.1\% and 67.5\% in predicting the surrounding physical fields and surface physical fields on OOD Reynolds, respectively.
In addition, MoNo also outperforms existing methods by at least 72.0\% and 68.6\% in predicting the surrounding physical fields and surface physical fields on OOD Angles, respectively.
Table~\ref{tab:ood_benchmarks} demonstrates the strong extrapolation capability of MoNo to unseen flow regimes and aerodynamic configurations.

\begin{table}[t]
  \centering
  \small
  \renewcommand\tabcolsep{5pt}
  \renewcommand{\arraystretch}{1.05}
  \begin{tabular}{l|cc}
    \toprule
    Model & Params. (M) & GFLOPs ($\downarrow$) \\
    \midrule
    GNOT~\cite{gnot} & 6.73 & 444.18 \\
    Transolver~\cite{transolver} & 2.81 & 196.19 \\
    Transolver++~\cite{transolverpp} & \underline{1.74} & 121.17 \\
    LinearNO~\cite{linearno} & 1.77 & 132.83 \\
    \midrule
    \textbf{MoNo-light (ours)} & \textbf{1.73} & \textbf{23.46} \\
    \textbf{MoNo (ours)} & 6.90 & \underline{97.26} \\
    \bottomrule
  \end{tabular}
    \caption{Efficiency comparison on parameter count (Params.) and GFLOPs ($\downarrow$). MoNo shows substantially higher computational efficiency than existing methods.}
    \label{tab:efficiency}
\end{table}

\begin{table}[!t]
  \centering
  \small
 \renewcommand\tabcolsep{5pt}
 \renewcommand{\arraystretch}{1.05}
 \begin{tabular}{cc|ccc}
    \toprule

    MultiScale & CoTAP & Airfoil & Elasticity & Pipe\\

    \midrule

    \ding{55} & \ding{55} & 0.0060 & 0.0077 & 0.0036 \\
    \ding{51} & \ding{55} & 0.0056 & 0.0072 & 0.0036 \\
    \ding{55} & \ding{51} & \underline{0.0051} & \underline{0.0070} & \underline{0.0028} \\
    \ding{51} & \ding{51} & \textbf{0.0048} & \textbf{0.0042} & \textbf{0.0027} \\
    
    \bottomrule
  \end{tabular}
  \caption{Ablation studies of CoTAP and progressive multiscale modeling in relative L2 error ($\downarrow$).}
  \label{tab:ablation_multi-mono}
\end{table}

\noindent \textbf{Efficiency.}
Table~\ref{tab:efficiency} compares the parameter counts and GFLOPs of models when processing 65,536 spatial observation points.
MoNo-light contains only 1.73M parameters and achieves the lowest computational cost.
Despite using a larger configuration with 6.90M parameters, MoNo requires only 97.26 GFLOPs, which is 19.7\% lower than that of Transolver++.
These results demonstrate that MoNo-light provides substantial advantages in both model size and computational cost, while MoNo increases representation capacity and solving performance while remaining more computationally efficient than existing methods.

\subsection{Ablation Studies}
Unless otherwise specified, ablation studies are conducted using MoNo-light.

\noindent \textbf{CoTAP.}
Table~\ref{tab:ablation_multi-mono} verifies the critical role of CoTAP.
Building upon the baseline with softmax-based projection, introducing CoTAP reduces the errors on Airfoil, Elasticity, and Pipe by 15.0\%, 9.1\%, and 22.2\%, respectively.
These results demonstrate that CoTAP effectively improves the learning of physical interactions in latent spaces by constructing balanced cross-space projections.

\noindent \textbf{Progressive Multiscale Modeling.}
Table~\ref{tab:ablation_multi-mono} also verifies the importance of progressive multiscale modeling in improving prediction performance.
Although introducing multiscale modeling alone with softmax-based projection provides only limited improvements, further incorporating progressive multiscale modeling upon CoTAP reduces the errors on Airfoil, Elasticity, and Pipe by 5.9\%, 40.0\%, and 3.6\%, respectively.
This demonstrates that the effectiveness of multiscale modeling highly depends on the stable construction of latent spaces.

\begin{figure}[!t]
    \centering
    \includegraphics[width=0.92\linewidth]{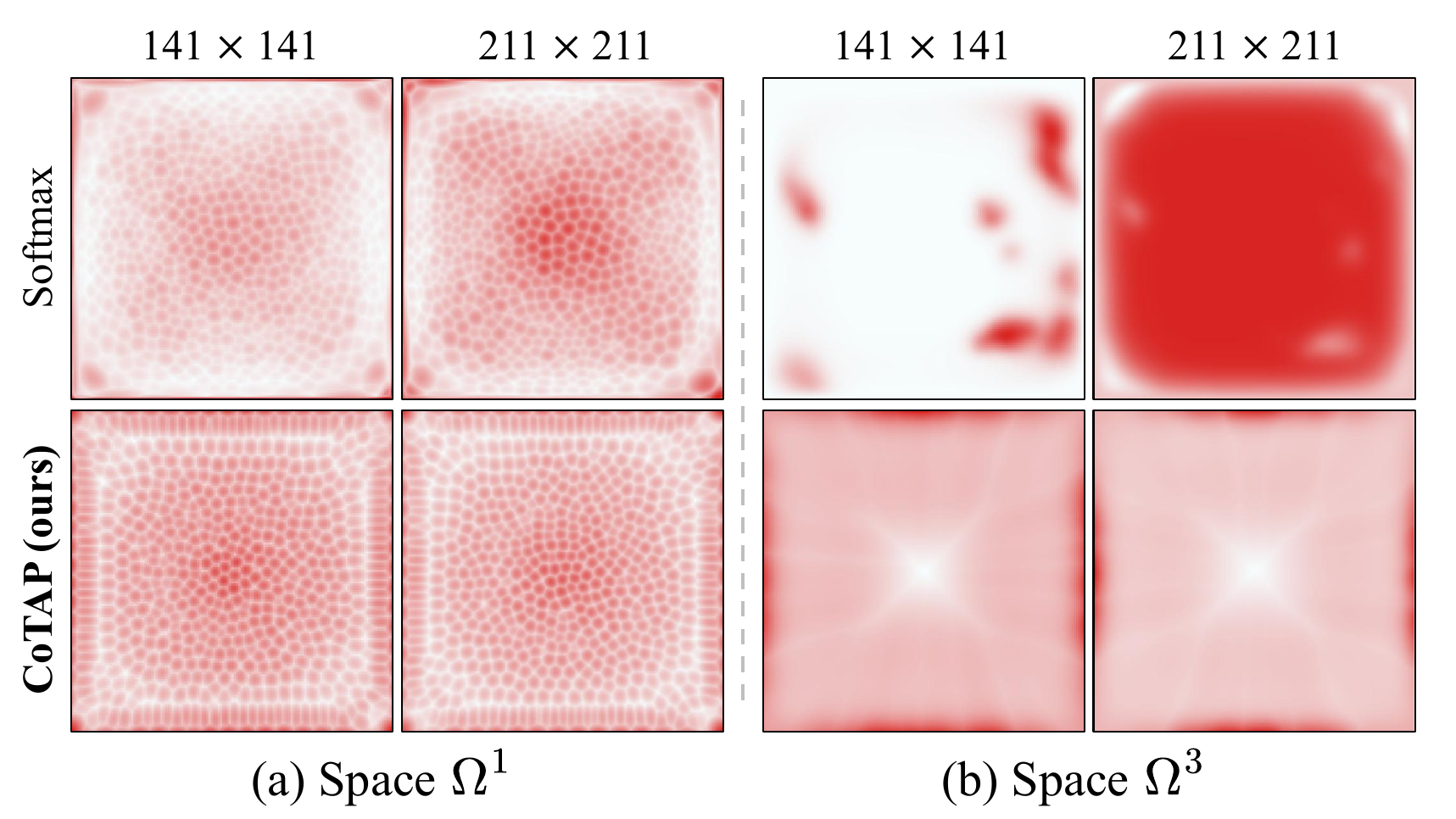}
    \caption{
    Visualization of assignments across different spatial resolutions.
    (a) The assignments in the first latent space (Space $\Omega^1$). (b) The assignments in the deeper latent space (Space $\Omega^3$).
    CoTAP constructs balanced assignment relations that remain consistent across spatial resolutions.
    }
    \label{fig:multi_resolution_assignment}
\end{figure}

\subsection{Visualizations}

\noindent \textbf{Visualization of CoTAP.}
Figures~\ref{fig:case_study}(a) and (b) visualize the assignment patterns learned by MoNo on NS2D and by MoNo-light on Elasticity, respectively.
Across both tasks with distinct geometries, CoTAP constructs assignment relations with clear multiscale characteristics.
Latent tokens in shallow spaces provide fine-grained partitions of the original physical domain, while each latent token progressively covers a larger spatial region and forms a more global representation as the hierarchy deepens.

Figure~\ref{fig:multi_resolution_assignment} further compares the assignment patterns generated by CoTAP and softmax-based projection on Darcy at resolutions of $141\times141$ and $211\times211$.
Across multiple latent spaces, CoTAP constructs balanced assignment patterns that remain consistent between the two spatial resolutions.
In contrast, softmax-based projection fails to preserve assignment consistency across resolutions, exhibits imbalanced assignments, and suffers from severe token collapse in deeper latent spaces.
These results validate the critical role of CoTAP in balancing token assignments, mitigating token collapse in deeper latent spaces, and preserving assignment stability across spatial resolutions.

\noindent \textbf{Case Studies.}
Figure~\ref{fig:case_study}(c) presents qualitative prediction results.
Compared with LinearNO, MoNo exhibits substantially lower prediction errors throughout the physical domain, particularly in high-gradient regions.
These results show that MoNo can accurately predict fine-grained physical behaviors and complex local variations.

\section{Conclusion}

In this work, we proposed MoNo, a progressive multiscale neural operator for solving PDEs on general geometries.
At its core, we propose CoTAP to formulate cross-space assignment as an entropy-regularized optimal transport problem and construct stable latent spaces.
CoTAP further unlocks stable progressive multiscale architectures on general geometries, thereby enabling more efficient learning of physical interactions.
Extensive experiments demonstrate that MoNo outperforms SOTA methods in both prediction performance and computational efficiency.
This work highlights the critical role of stable latent-space construction in developing Transformer-based neural operators with stronger representational capacity and higher computational efficiency.

\appendix

\section{Appendix A: Notation Summary}

For clarity, Table~\ref{tab:notation} summarizes the main notation used throughout this paper.
Symbols that are only used in Appendix E are defined upon their first occurrence.

\begin{table*}[!t]
\centering
\small
\renewcommand{\arraystretch}{1.13}
\begin{tabular}{@{}p{0.20\textwidth}p{0.80\textwidth}@{}}
\toprule
\textbf{Notation} & \textbf{Description} \\
\midrule

\multicolumn{2}{@{}l}{\textit{Dimensions and indices}} \\
$D_s$ & Spatial dimension of the original physical domain. \\
$D_x$ & Dimension of the input observation at each observation point. \\
$D_y$ & Dimension of the target physical field at each observation point. \\
$D_g$ & Dimension of the anchor tokens. \\
$D_f$ & Dimension of the physical state tokens. \\
$N^0$ & Number of observation points in the original observation space. \\
$N^l$ & Number of elements in $\Omega^l$. For $l\geq1$, the number of latent tokens in the $l$-th latent space. \\
$L$ & Total number of latent spaces in the progressive hierarchy. \\
$l$ & Space index, where $l=0$ denotes the original observation space and $l\in\{1,\ldots,L\}$ denotes a latent space. \\

\midrule
\multicolumn{2}{@{}l}{\textit{Spaces, observations, and tokens}} \\
$\Omega^l$ & The $l$-th space in the hierarchy, where $\Omega^0\subset\mathbb{R}^{D_s}$ is the original physical domain and $\Omega^l$ for $l\geq1$ is the $l$-th latent space. \\
$P=\{\mathbf p_i\}_{i=1}^{N^0}$ & Spatial observation point set, where $\mathbf p_i\in\Omega^0$. \\
$X=\{\mathbf x_i\}_{i=1}^{N^0}$ & Input observation set, where $\mathbf x_i\in\mathbb{R}^{D_x}$. \\
$Y=\{\mathbf y_i\}_{i=1}^{N^0}$ & Ground-truth target physical field, where $\mathbf y_i\in\mathbb{R}^{D_y}$. \\
$\hat{Y}=\{\hat{\mathbf y}_i\}_{i=1}^{N^0}$ & Discrete target physical field predicted by MoNo. \\
$G^l=\{\mathbf g_i^l\}_{i=1}^{N^l}$ & Anchor token set in $\Omega^l$, where $\mathbf g_i^l\in\mathbb{R}^{D_g}$. \\
$F_{\mathrm{enc}}^l$ & Encoded physical state token set in $\Omega^l$, with $F_{\mathrm{enc}}^l\in\mathbb{R}^{N^l\times D_f}$. \\
$F_{\mathrm{dec}}^l$ & Decoded physical state token set in $\Omega^l$, with $F_{\mathrm{dec}}^l\in\mathbb{R}^{N^l\times D_f}$. \\

\midrule
\multicolumn{2}{@{}l}{\textit{Trainable modules}} \\
$\mathcal{P}_{\mathrm{anchor}}$ & Anchor embedding module that maps observation positions to the initial anchor tokens. \\
$\mathcal{P}_{\mathrm{phy}}$ & Physical state embedding module that maps input observations to the initial physical state tokens. \\
$\mathcal{P}^l$ & MLP assignment projector that maps $G^{l-1}$ to the initial assignment scores for $\Omega^l$. \\
$\mathcal{M}_{\mathrm{enc}}^l$ & Encoding Transformer block that learns physical interactions in $\Omega^l$. \\
$\mathcal{M}_{\mathrm{dec}}^l$ & Decoding Transformer block that updates and fuses decoded physical state tokens in $\Omega^l$. \\

\midrule
\multicolumn{2}{@{}l}{\textit{CoTAP}} \\
$\mathbf S_{\mathrm{init}}^{l-1,l}$ & Initial assignment matrix generated by $\mathcal P^l$, with $\mathbf S_{\mathrm{init}}^{l-1,l}\in\mathbb{R}^{N^{l-1}\times N^l}$. \\
$\mathbf S_{\mathrm{init},ij}^{l-1,l}$ & Initial assignment score between the $i$-th element in $\Omega^{l-1}$ and the $j$-th latent token in $\Omega^l$. \\
$\mathbf S_{\mathrm{ot}}^{l-1,l}$ & OT-normalized assignment matrix between $\Omega^{l-1}$ and $\Omega^l$ with uniform marginal constraints. \\
$\mathbf S_{\mathrm{enc}}^{l-1,l}$ & Encoding projection matrix from $\Omega^{l-1}$ to $\Omega^l$. \\
$\mathbf S_{\mathrm{dec}}^{l-1,l}$ & Decoding projection matrix from $\Omega^l$ to $\Omega^{l-1}$. \\
$\mathbf S$ & Candidate transport matrix in the entropy-regularized optimal transport problem. \\
$\mathcal C^{l-1,l}$ & Transport polytope with uniform marginal constraints between $\Omega^{l-1}$ and $\Omega^l$. \\
$T$ & Number of Sinkhorn iterations used to solve the entropy-regularized optimal transport problem. \\
$\tau$ & Temperature coefficient of the entropy-regularized optimal transport problem. \\
$\mathcal H(\mathbf S)$ & Entropy regularization term of the candidate transport matrix $\mathbf S$. \\

\midrule
\multicolumn{2}{@{}l}{\textit{Continuous operator analysis}} \\
$\mu^l$ & Probability measure defined on $\Omega^l$. \\
$\mathcal G^l$ & Continuous anchor token field associated with the discrete anchor token set $G^l$. \\
$\mathcal F_{\mathrm{enc}}^l,\mathcal F_{\mathrm{dec}}^l$ & Continuous encoded and decoded physical state token fields on $\Omega^l$. \\
$\mathcal S_{\mathrm{enc}}^{l-1,l},\mathcal S_{\mathrm{dec}}^{l-1,l}$ & Continuous CoTAP encoding and decoding operators between adjacent spaces. \\
$r_{\mathcal G^{l-1}}^{l-1,l}$ & CoTAP coupling density induced by the anchor token field $\mathcal G^{l-1}$. \\
$\mathcal A^1,\alpha$ & Latent self-attention operator on $\Omega^1$ and its input-dependent attention kernel. \\
$\mathcal T$ & Composite operator formed by CoTAP and a single latent self-attention layer. \\
$\kappa$ & Input-dependent integral kernel induced by CoTAP and latent self-attention. \\
$\mathbf W_v$ & Value projection matrix in the integral representation of latent self-attention. \\
$\mathcal N_\theta$ & MoNo neural operator parameterized by learnable parameters $\theta$. \\
$\widehat{\mathcal Y}$ & Continuous target physical-field function predicted by MoNo on $\Omega^0$. \\

\bottomrule
\end{tabular}
\caption{Summary of the main notation used in MoNo.}
\label{tab:notation}
\end{table*}

\section{Appendix B: Proof of Theorem 1}
\label{app:proof_cotap_integral_operator_v2}

\begin{proof}
Assume that the continuous CoTAP coupling between
$(\Omega^0,\mu^0)$ and $(\Omega^1,\mu^1)$ admits a nonnegative
measurable density with respect to $\mu^0\otimes\mu^1$, denoted as:
\begin{equation}
r(\mathbf{p},\boldsymbol{\zeta})
:=
    r_{\mathcal{G}^0}^{0,1}
(\mathbf{p},\boldsymbol{\zeta}),
\label{eq:thm1_v2_coupling_density}
\end{equation}
where $\mu^1$ denotes the probability measure on the latent space $\Omega^1$, $r$ denotes the CoTAP coupling density induced by the anchor token field $\mathcal{G}^0$, and $\mathbf{p}$ denotes the target point in $\Omega^0$. The marginal constraints are as follows:
\begin{equation}
\begin{aligned}
\int_{\Omega^1}
r(\mathbf{p},\boldsymbol{\zeta})
\,d\mu^1(\boldsymbol{\zeta})
&=1,
&&\mu^0\text{-a.e. }\mathbf{p},\\
\int_{\Omega^0}
r(\boldsymbol{\xi},\boldsymbol{\zeta})
\,d\mu^0(\boldsymbol{\xi})
&=1,
&&\mu^1\text{-a.e. }\boldsymbol{\zeta}.
\end{aligned}
\label{eq:thm1_v2_cotap_marginals}
\end{equation}

\noindent \textbf{Continuous CoTAP.}
The continuous CoTAP encoding and decoding operators induced by this coupling are as follows:
\begin{equation}
\begin{aligned}
\mathcal{F}^1(\boldsymbol{\zeta})
&:=
\left(
    \mathcal{S}_{\mathrm{enc}}^{0,1}
    \mathcal{F}_{\mathrm{enc}}^0
\right)(\boldsymbol{\zeta})
\\
&=
\int_{\Omega^0}
r(\boldsymbol{\xi},\boldsymbol{\zeta})
\mathcal{F}_{\mathrm{enc}}^0(\boldsymbol{\xi})
\,d\mu^0(\boldsymbol{\xi}),\\[2pt]
\left(
    \mathcal{S}_{\mathrm{dec}}^{0,1}
    \Phi^1
\right)(\mathbf{p})
&:=
\int_{\Omega^1}
r(\mathbf{p},\boldsymbol{\zeta})
\Phi^1(\boldsymbol{\zeta})
\,d\mu^1(\boldsymbol{\zeta}),
\end{aligned}
\label{eq:thm1_v2_continuous_cotap}
\end{equation}
where $\mathcal{F}^1$ denotes the physical state field projected into $\Omega^1$, and $\Phi^1$ denotes a measurable latent feature field on $\Omega^1$.

\noindent \textbf{Latent self-attention.}
As established in prior work~\cite{kovachki2023neural},
standard self-attention admits an integral-operator representation.
Denoting its input-dependent attention kernel by
$\alpha(\boldsymbol{\zeta},\boldsymbol{\eta})$, the latent
self-attention operator in the present setting is:
\begin{equation}
\begin{aligned}
\left(
\mathcal{A}^1\mathcal{F}^1
\right)(\boldsymbol{\zeta})
=
\int_{\Omega^1}
\alpha(\boldsymbol{\zeta},\boldsymbol{\eta})
\mathcal{F}^1(\boldsymbol{\eta})
\mathbf{W}_v
\,d\mu^1(\boldsymbol{\eta}),
\end{aligned}
\label{eq:thm1_v2_latent_attention}
\end{equation}
where $\mathcal{A}^1$ denotes the latent self-attention operator on $\Omega^1$. The attention kernel satisfies:
\begin{equation}
\alpha(\boldsymbol{\zeta},\boldsymbol{\eta})\geq0,
\qquad
\int_{\Omega^1}
\alpha(\boldsymbol{\zeta},\boldsymbol{\eta})
\,d\mu^1(\boldsymbol{\eta})=1.
\label{eq:thm1_v2_attention_normalization}
\end{equation}

\noindent \textbf{Composition on the original space.}
To represent the composition of latent self-attention and CoTAP decoding, we first define an intermediate kernel.
Specifically, $\alpha(\boldsymbol{\zeta},\boldsymbol{\eta})$ transfers the physical state at the latent point $\boldsymbol{\eta}$ to the attention output at $\boldsymbol{\zeta}$, while $r(\mathbf{p},\boldsymbol{\zeta})$ further projects this output from $\boldsymbol{\zeta}$ to the original-space point $\mathbf{p}$.
Marginalizing over the intermediate latent point $\boldsymbol{\zeta}$ yields the effective transfer kernel from $\boldsymbol{\eta}$ to $\mathbf{p}$:
\begin{equation}
\beta(\mathbf{p},\boldsymbol{\eta})
:=
\int_{\Omega^1}
r(\mathbf{p},\boldsymbol{\zeta})
\alpha(\boldsymbol{\zeta},\boldsymbol{\eta})
\,d\mu^1(\boldsymbol{\zeta}).
\label{eq:thm1_v2_intermediate_kernel}
\end{equation}
$\beta$ is the intermediate kernel induced by the composition of latent self-attention and CoTAP decoding.

Substituting the intermediate kernel $\beta$ into the composite operator and then applying the continuous CoTAP encoding gives, for $\mu^0$-almost every $\mathbf{p}\in\Omega^0$:
\begin{equation}
\begin{aligned}
&\mathcal{T}
\bigl(\mathcal{G}^0,\mathcal{F}_{\mathrm{enc}}^0\bigr)
(\mathbf{p})
\\
&=
\int_{\Omega^1}
\beta(\mathbf{p},\boldsymbol{\eta})
\mathcal{F}^1(\boldsymbol{\eta})
\mathbf{W}_v
\,d\mu^1(\boldsymbol{\eta})
\\
&=
\int_{\Omega^1}\int_{\Omega^0}
\beta(\mathbf{p},\boldsymbol{\eta})
r(\boldsymbol{\xi},\boldsymbol{\eta})
\mathcal{F}_{\mathrm{enc}}^0(\boldsymbol{\xi})
\mathbf{W}_v
\,d\mu^0(\boldsymbol{\xi})
\,d\mu^1(\boldsymbol{\eta}).
\end{aligned}
\label{eq:thm1_v2_decode_attention}
\end{equation}

To express this double integral as an integral solely over the original space $\Omega^0$, we first verify the absolute integrability required by Fubini's theorem:
\begin{equation}
\begin{aligned}
&\int_{\Omega^1}\int_{\Omega^0}
\beta(\mathbf{p},\boldsymbol{\eta})
r(\boldsymbol{\xi},\boldsymbol{\eta})
\\[-2pt]
&\qquad\cdot
\left\|
\mathcal{F}_{\mathrm{enc}}^0(\boldsymbol{\xi})
\mathbf{W}_v
\right\|
\,d\mu^0(\boldsymbol{\xi})
\,d\mu^1(\boldsymbol{\eta})
\\
&\leq
\left\|\mathcal{F}_{\mathrm{enc}}^0\right\|_{L^\infty(\mu^0)}
\left\|\mathbf{W}_v\right\|
\\[-2pt]
&\qquad\cdot
\int_{\Omega^1}
\beta(\mathbf{p},\boldsymbol{\eta})
\left[
\int_{\Omega^0}
r(\boldsymbol{\xi},\boldsymbol{\eta})
\,d\mu^0(\boldsymbol{\xi})
\right]
\,d\mu^1(\boldsymbol{\eta})
\\
&=
\left\|\mathcal{F}_{\mathrm{enc}}^0\right\|_{L^\infty(\mu^0)}
\left\|\mathbf{W}_v\right\|
<\infty.
\end{aligned}
\label{eq:thm1_v2_absolute_integrability}
\end{equation}
The equality follows from
Equations~\eqref{eq:thm1_v2_cotap_marginals},
\eqref{eq:thm1_v2_attention_normalization}, and the definition of
$\beta$ in Equation~\eqref{eq:thm1_v2_intermediate_kernel}.
Therefore, Fubini's theorem allows us to exchange the order of integration:
\begin{equation}
\begin{aligned}
&\mathcal{T}
\bigl(\mathcal{G}^0,\mathcal{F}_{\mathrm{enc}}^0\bigr)
(\mathbf{p})
\\
&=
\int_{\Omega^0}
\left[
\int_{\Omega^1}
\beta(\mathbf{p},\boldsymbol{\eta})
r(\boldsymbol{\xi},\boldsymbol{\eta})
\,d\mu^1(\boldsymbol{\eta})
\right]
\\
&\qquad\cdot
\mathcal{F}_{\mathrm{enc}}^0(\boldsymbol{\xi})
\mathbf{W}_v
\,d\mu^0(\boldsymbol{\xi}).
\end{aligned}
\label{eq:thm1_v2_fubini_reordering}
\end{equation}

The inner integral in Equation~\eqref{eq:thm1_v2_fubini_reordering} defines the effective original-space kernel:
\begin{equation}
\kappa(\mathbf{p},\boldsymbol{\xi})
:=
\int_{\Omega^1}
\beta(\mathbf{p},\boldsymbol{\eta})
r(\boldsymbol{\xi},\boldsymbol{\eta})
\,d\mu^1(\boldsymbol{\eta}),
\label{eq:thm1_v2_original_kernel}
\end{equation}
It follows that, for $\mu^0$-almost every $\mathbf{p}\in\Omega^0$, we can get:
\begin{equation}
\begin{aligned}
&\mathcal{T}
\bigl(\mathcal{G}^0,\mathcal{F}_{\mathrm{enc}}^0\bigr)(\mathbf{p})
=
\int_{\Omega^0}
\kappa(\mathbf{p},\boldsymbol{\xi})
\mathcal{F}_{\mathrm{enc}}^0(\boldsymbol{\xi})
\mathbf{W}_v
\,d\mu^0(\boldsymbol{\xi}),
\end{aligned}
\end{equation}
which proves the integral-operator representation on $\Omega^0$.
\end{proof}

\section{Appendix C: Discrete Kernel Representation of Theorem 1}
\label{app:discrete_cotap_kernel}

We further provide the finite-dimensional counterpart of the operator considered
in Theorem 1.
Using the CoTAP projection matrices defined in the main paper, we show
that the composition of CoTAP encoding, one pure self-attention layer,
and CoTAP decoding admits an exact input-dependent kernel-matrix
representation on the discrete observation space.
The encoded latent features are:
\begin{equation}
\begin{aligned}
F^1
&=
    \left(\mathbf{S}_{\mathrm{enc}}^{0,1}\right)^{\top}
F_{\mathrm{enc}}^0.
\end{aligned}
\label{eq:discrete_cotap_encoded_features}
\end{equation}

In addition, the latent self-attention is defined as:
\begin{equation}
\begin{aligned}
\mathbf{Q}
&=F^1\mathbf{W}_q,
&\\
\mathbf{K}
&=F^1\mathbf{W}_k,
\\
\mathbf{A}(F^1)
&=
\operatorname{Softmax}
\left(
\frac{\mathbf{Q}\mathbf{K}^{\top}}{\sqrt{D_a}}
\right),
\end{aligned}
\label{eq:discrete_latent_attention}
\end{equation}
where $\mathbf{Q}$ and $\mathbf{K}$ denote the query and key feature matrices, $\mathbf{A}(\cdot)$ denotes the row-normalized attention matrix based on input, $\mathbf{W}_q$ and $\mathbf{W}_k$ denote the query and key projections, and $D_a$ denotes the query/key dimension.
The corresponding latent attention output is
$\mathbf{A}(F^1)F^1\mathbf{W}_v$.
The projection matrices and the resulting CoTAP decoding are as follows:
\begin{equation}
\begin{aligned}
\mathbf{S}_{\mathrm{enc}}^{0,1}
&=
N^1\mathbf{S}_{\mathrm{ot}}^{0,1},\\
\mathbf{S}_{\mathrm{dec}}^{0,1}
&=
N^0\mathbf{S}_{\mathrm{ot}}^{0,1},\\
F_{\mathrm{dec}}^0
&=
    \mathbf{S}_{\mathrm{dec}}^{0,1}
\mathbf{A}(F^1)\cdot
    \left(\mathbf{S}_{\mathrm{enc}}^{0,1}\right)^{\top}
F_{\mathrm{enc}}^0
\mathbf{W}_v,
\end{aligned}
\label{eq:discrete_cotap_attention_composition_v2}
\end{equation}
where $F_{\mathrm{dec}}^0$ denotes the physical state feature matrix projected back to the original observation space after the composition of CoTAP encoding, latent self-attention, and CoTAP decoding.
Therefore, the effective discrete kernel is:
\begin{equation}
\begin{aligned}
&\mathbf{K}_{\mathrm{eff}}^0
\left(G^0,F_{\mathrm{enc}}^0\right)
\\
&:=
    \mathbf{S}_{\mathrm{dec}}^{0,1}
\mathbf{A}(F^1)
    \left(\mathbf{S}_{\mathrm{enc}}^{0,1}\right)^{\top}
\\
&=
\left(
N^0\mathbf{S}_{\mathrm{ot}}^{0,1}
\right)
\mathbf{A}(F^1)
\left(
N^1\mathbf{S}_{\mathrm{ot}}^{0,1}
\right)^{\top}
\\
&=
N^0N^1
    \mathbf{S}_{\mathrm{ot}}^{0,1}
\mathbf{A}
\left(
N^1
    \left(\mathbf{S}_{\mathrm{ot}}^{0,1}\right)^{\top}
F_{\mathrm{enc}}^0
\right)
\cdot
    \left(\mathbf{S}_{\mathrm{ot}}^{0,1}\right)^{\top},
\end{aligned}
\label{eq:discrete_effective_kernel_v2}
\end{equation}
where $\mathbf{K}_{\mathrm{eff}}^0\in\mathbb{R}^{N^0\times N^0}$
acts directly on the original observation space.
The complete discrete composition can therefore be written as:
\begin{equation}
F_{\mathrm{dec}}^0
=
\mathbf{K}_{\mathrm{eff}}^0
\left(G^0,F_{\mathrm{enc}}^0\right)
F_{\mathrm{enc}}^0
\mathbf{W}_v.
\label{eq:discrete_kernel_identity_v2}
\end{equation}
This is an exact algebraic identity for finite $N^0$ and $N^1$.
The kernel is input-dependent because
$\mathbf{S}_{\mathrm{ot}}^{0,1}$ is induced by $G^0$, while
$\mathbf{A}(\cdot)$ depends on $F^1$, which is jointly determined by
$G^0$ and $F_{\mathrm{enc}}^0$.

\section{Appendix D: Proof of Theorem 2}
\label{app:proof_mono_neural_operator_v2}

\begin{proof}
We consider the deterministic inference map of MoNo. For every $l=1,\ldots,L$, assume that the CoTAP coupling
between $(\Omega^{l-1},\mu^{l-1})$ and $(\Omega^l,\mu^l)$ admits a
nonnegative measurable density with respect to
$\mu^{l-1}\otimes\mu^l$.
By the continuous CoTAP construction in
Equation~\eqref{eq:thm1_v2_continuous_cotap},
    $\mathcal{S}_{\mathrm{enc}}^{l-1,l}$ and
    $\mathcal{S}_{\mathrm{dec}}^{l-1,l}$ are cross-space
integral operators from $\Omega^{l-1}$ to $\Omega^l$ and from
$\Omega^l$ to $\Omega^{l-1}$, respectively.
Both operators depend on the current anchor token field
$\mathcal{G}^{l-1}$.

As shown in Equation~\eqref{eq:thm1_v2_latent_attention}, latent self-attention is a nonlocal integral operator on the corresponding latent space.
Multi-head concatenation and output projection act only along the feature dimension, while LayerNorm, feed-forward networks, and activation functions act point-wise on each feature field.
Moreover, residual and skip addition can be written as the point-wise local operators.
Thus, after lifting branched features to a product function space, every standard Transformer block is a finite composition of latent-space nonlocal integral operators and point-wise local operators.
Specifically, each $\mathcal{M}_{\mathrm{enc}}^l$ and
$\mathcal{M}_{\mathrm{dec}}^l$ has the same operator structure.

\noindent \textbf{Single latent space.}
For $L=1$, the complete model is as follows:
\begin{equation}
\begin{aligned}
\mathcal{F}_{\mathrm{enc}}^1
&=
\mathcal{M}_{\mathrm{enc}}^1
\left(
    \mathcal{S}_{\mathrm{enc}}^{0,1}
\mathcal{F}_{\mathrm{enc}}^0
\right),\\
\mathcal{F}_{\mathrm{dec}}^1
&=
\mathcal{M}_{\mathrm{dec}}^1
\left(\mathcal{F}_{\mathrm{enc}}^1\right),\\
\mathcal{F}_{\mathrm{dec}}^0
&=
    \mathcal{S}_{\mathrm{dec}}^{0,1}
\mathcal{F}_{\mathrm{dec}}^1,\\
\widehat{\mathcal{Y}}
&=
\mathrm{MLP}_{\mathrm{out}}
\left(\mathcal{F}_{\mathrm{dec}}^0\right).
\end{aligned}
\label{eq:thm2_v2_single_space}
\end{equation}
The CoTAP operators in Equation~\eqref{eq:thm2_v2_single_space} are
induced by $\mathcal{G}^0$.
Hence, these equations define a learnable mapping
$(\mathcal{G}^0,\mathcal{F}_{\mathrm{enc}}^0)\mapsto
\widehat{\mathcal{Y}}$ from input fields on $\Omega^0$ to an output field on $\Omega^0$.
By the operator characterization above, this mapping is a finite composition of cross-space integral, latent-space nonlocal, and point-wise local operators. The claim therefore holds for a single latent space.

\paragraph{Multiple latent spaces.}
For $l=1,\ldots,L$, the progressive encoding recursion is:
\begin{equation}
\begin{aligned}
\mathcal{F}_{\mathrm{enc}}^l
&=
\mathcal{M}_{\mathrm{enc}}^l
\left(
    \mathcal{S}_{\mathrm{enc}}^{l-1,l}
\mathcal{F}_{\mathrm{enc}}^{l-1}
\right),\\
\mathcal{G}^l
&=
    \mathcal{S}_{\mathrm{enc}}^{l-1,l}
\mathcal{G}^{l-1}.
\end{aligned}
\label{eq:thm2_v2_multiscale_encoding}
\end{equation}
The base case $l=1$ follows from the single-space construction.
If $(\mathcal{G}^{l-1},\mathcal{F}_{\mathrm{enc}}^{l-1})$ is obtained
from the original input through the permitted operators, then
Equation~\eqref{eq:thm2_v2_multiscale_encoding} applies one cross-space
integral operator followed by the encoder model in $\Omega^l$.
Therefore, by forward induction, every
$(\mathcal{G}^l,\mathcal{F}_{\mathrm{enc}}^l)$ is a function field on
$\Omega^l$ produced by a finite composition of the same operators.
The progressive decoder is:
\begin{equation}
\mathcal{F}_{\mathrm{dec}}^l
=
\begin{cases}
\mathcal{M}_{\mathrm{dec}}^L
\left(\mathcal{F}_{\mathrm{enc}}^L\right),
& l=L,\\
\mathcal{M}_{\mathrm{dec}}^l
(\mathcal{F}_{\mathrm{enc}}^l+\mathcal{S}_{\mathrm{dec}}^{l,l+1}
\mathcal{F}_{\mathrm{dec}}^{l+1}),
& 1 \leq l < L.
\end{cases}
\label{eq:thm2_v2_multiscale_decoding}
\end{equation}
At the deepest latent space, $\mathcal{F}_{\mathrm{dec}}^L$ is obtained by applying the decoding Transformer model $\mathcal{M}_{\mathrm{dec}}^L$ to $\mathcal{F}_{\mathrm{enc}}^L$.
Therefore, $\mathcal{F}_{\mathrm{dec}}^L$ has the operator structure characterized above.
The projected field is then combined with $\mathcal{F}_{\mathrm{enc}}^l$ through point-wise skip addition and subsequently updated by $\mathcal{M}_{\mathrm{dec}}^l$, which is a finite composition of latent-space nonlocal and point-wise local operators.
Therefore, backward induction from $l=L$ to $l=1$ shows that every decoded field $\mathcal{F}_{\mathrm{dec}}^l$ is produced by a finite composition of these operators.
Then, the predicted physical-field function $\widehat{\mathcal{Y}}$ is obtained as:
\begin{equation}
\begin{aligned}
\mathcal{F}_{\mathrm{dec}}^0
&=
    \mathcal{S}_{\mathrm{dec}}^{0,1}
\mathcal{F}_{\mathrm{dec}}^1,\\
\widehat{\mathcal{Y}}
&=
\mathrm{MLP}_{\mathrm{out}}
\left(\mathcal{F}_{\mathrm{dec}}^0\right).
\end{aligned}
\label{eq:thm2_v2_output}
\end{equation}
$\mathcal{S}_{\mathrm{dec}}^{0,1}$ maps $\mathcal{F}_{\mathrm{dec}}^1$ back to the original space $\Omega^0$ through a cross-space integral operator, while $\mathrm{MLP}_{\mathrm{out}}$ transforms the resulting field into the target physical field through point-wise linear mappings and nonlinear activations.
Consequently, $\widehat{\mathcal{Y}}$ is a function on $\Omega^0$ obtained through a finite composition of cross-space integral, latent-space nonlocal, and point-wise local operators.
Combining the forward and backward inductions, MoNo defines:
\begin{equation}
\mathcal{N}_{\theta}:
\left(
\mathcal{G}^0,
\mathcal{F}_{\mathrm{enc}}^0
\right)
\longmapsto
\widehat{\mathcal{Y}}.
\label{eq:thm2_v2_complete_operator}
\end{equation}

Combining the encoding and decoding processes, $\mathcal{N}_{\theta}$ is a finite composition of cross-space integral, latent-space nonlocal, and point-wise local operators.
Therefore, MoNo constitutes a neural operator that maps the original input functions to the target physical-field function on $\Omega^0$.
\end{proof}

\section{Appendix E: CoTAP Iteration}

\subsection{Standard Log-domain Sinkhorn Iteration}
We first describe the standard log-domain Sinkhorn iteration~\cite{sinkhorn1967diagonal} for solving the entropy-regularized optimal transport problem in CoTAP.
Let $\overline{\mathbf{S}}_{\mathrm{init}}^{l-1,l}=\mathbf{S}_{\mathrm{init}}^{l-1,l}/\tau$, and initialize the row and column dual variables $\mathbf{u}^{0}\in\mathbb{R}^{N^{l-1}}$ and $\mathbf{v}^{0}\in\mathbb{R}^{N^l}$ as zero vectors.
Since CoTAP adopts uniform marginal constraints, the row and column marginals in the log domain are $-\log N^{l-1}$ and $-\log N^l$, respectively.
Therefore, at the $t$-th iteration, $\mathbf{u}^{t+1}$ and $\mathbf{v}^{t+1}$ are updated as:
\begin{equation}
\begin{aligned}
    u_i^{t+1}
    &=
    -\log N^{l-1}
    -
    \log\sum_{j=1}^{N^l}
    \exp\left(
        \overline{S}_{\mathrm{init},ij}^{l-1,l}
        +
        v_j^{t}
    \right),\\
    v_j^{t+1}
    &=
    -\log N^l
    -
    \log\sum_{i=1}^{N^{l-1}}
    \exp\left(
        \overline{S}_{\mathrm{init},ij}^{l-1,l}
        +
        u_i^{t+1}
    \right),
\end{aligned}
\label{eq:sinkhorn_update}
\end{equation}
where, $t\in\{0,\ldots,T-1\}$ denotes the iteration index, and $T$ denotes the total number of iterations.
After $T$ iterations, CoTAP obtains the OT-normalized assignment matrix:
\begin{equation}
    S_{\mathrm{ot},ij}^{l-1,l}
    =
    \exp\left(
        \overline{S}_{\mathrm{init},ij}^{l-1,l}
        +
        u_i^{T}
        +
        v_j^{T}
    \right).
\label{eq:sinkhorn_assignment}
\end{equation}

Equation~\eqref{eq:sinkhorn_update} alternately enforces the uniform row and column marginals in the log domain, yielding a balanced assignment.
However, the standard implementation produces large intermediate tensors at each iteration and retains the complete computation graph for automatic differentiation during training, resulting in substantial GPU memory consumption for large-scale observations.
To address this issue, we further develop a Triton-based fused Sinkhorn solver, termed CoTAP Iteration.

\subsection{Fused Forward in CoTAP Iteration}

CoTAP Iteration implements the row and column dual updates at each iteration of Equation~\eqref{eq:sinkhorn_update} as fused Triton kernels.
Each Triton program processes one row or column in parallel, performs assignment scaling, log-sum-exp reduction, and dual-variable update using on-chip memory, and writes only the updated dual variable back to GPU memory.
After $T$ iterations, a fused output kernel generates $\mathbf{S}_{\mathrm{enc}}^{l-1,l}$ and $\mathbf{S}_{\mathrm{dec}}^{l-1,l}$ without separately materializing $\mathbf{S}_{\mathrm{ot}}^{l-1,l}$.

The fused forward avoids materializing intermediate tensors of size $N^{l-1}\times N^l$ at each iteration and caches only the dual variables required for the backward pass.

\subsection{Fused Backward in CoTAP Iteration}

End-to-end training of the assignment projector $\mathcal{P}^l(\cdot)$ requires propagating the gradient from $\mathbf{S}_{\mathrm{enc}}^{l-1,l}$ and $\mathbf{S}_{\mathrm{dec}}^{l-1,l}$ back to the initial assignment matrix $\mathbf{S}_{\mathrm{init}}^{l-1,l}$.
In the standard log-domain Sinkhorn implementation, automatic differentiation computes this gradient by retaining the complete computation graph of all Sinkhorn iterations.
However, the fused forward incorporates assignment reconstruction, log-sum-exp reductions, and dual-variable updates into Triton kernels, whose internal operations are not recorded by the automatic differentiation system.
Therefore, we explicitly reverse the finite Sinkhorn iterations and accumulate reverse-mode gradients using the dual-variable histories cached during the forward pass.

As $\mathbf{S}_{\mathrm{enc}}^{l-1,l}$ and $\mathbf{S}_{\mathrm{dec}}^{l-1,l}$ are outputs of the fused forward, their upstream gradients
$\partial\mathcal{L}/\partial\mathbf{S}_{\mathrm{enc}}^{l-1,l}$ and
$\partial\mathcal{L}/\partial\mathbf{S}_{\mathrm{dec}}^{l-1,l}$
are directly available at the backward pass, where $\mathcal{L}$ denotes the training loss.
The fused backward first reverses the fused output operation and accumulates the gradients from both projection outputs into the final log-assignment.
It then traverses the Sinkhorn iterations in reverse order, propagating gradients through each column dual update followed by the corresponding row dual update.
This procedure yields the same gradient with respect to $\mathbf{S}_{\mathrm{init}}^{l-1,l}$ as automatic differentiation through the standard finite-step log-domain Sinkhorn iterations.

We first describe the reverse-mode gradient accumulation with zero and one Sinkhorn iteration and then generalize it to an arbitrary number of iterations $T$.
For simplicity, we denote the final log-assignment after $T$ iterations by $\mathbf{Z}^{l-1, l}$:
\begin{equation}
    Z_{ij}^{l-1, l}
    =
    \overline{S}_{\mathrm{init},ij}^{l-1,l}
    +
    u_i^{T}
    +
    v_j^{T}.
\label{eq:sinkhorn_final_log_assignment}
\end{equation}

Since
$\mathbf{S}_{\mathrm{enc}}^{l-1,l}
=N^l\mathbf{S}_{\mathrm{ot}}^{l-1,l}$
and
$\mathbf{S}_{\mathrm{dec}}^{l-1,l}
=N^{l-1}\mathbf{S}_{\mathrm{ot}}^{l-1,l}$,
the upstream gradients from the two projection outputs are accumulated at
$\mathbf{Z}^{l-1,l}$:
\begin{equation}
\begin{aligned}
    \frac{\partial\mathcal{L}}
    {\partial Z_{ij}^{l-1,l}}
    &=
    \frac{\partial\mathcal{L}}
    {\partial S_{\mathrm{enc},ij}^{l-1,l}}
    \frac{\partial S_{\mathrm{enc},ij}^{l-1,l}}
    {\partial Z_{ij}^{l-1,l}}
    +
    \frac{\partial\mathcal{L}}
    {\partial S_{\mathrm{dec},ij}^{l-1,l}}
    \frac{\partial S_{\mathrm{dec},ij}^{l-1,l}}
    {\partial Z_{ij}^{l-1,l}}\\
    &=
    \left(
        N^l
        \frac{\partial\mathcal{L}}
        {\partial S_{\mathrm{enc},ij}^{l-1,l}}
        +
        N^{l-1}
        \frac{\partial\mathcal{L}}
        {\partial S_{\mathrm{dec},ij}^{l-1,l}}
    \right)
    S_{\mathrm{ot},ij}^{l-1,l}.
\end{aligned}
\label{eq:sinkhorn_output_reverse}
\end{equation}
In the fused backward, neither $\mathbf{Z}^{l-1,l}$ nor
$\mathbf{S}_{\mathrm{ot}}^{l-1,l}$ is cached as a complete tensor during the forward pass.
Instead, both quantities are reconstructed element-wise on demand from
$\mathbf{S}_{\mathrm{init}}^{l-1,l}$ within the backward kernel and are discarded immediately after evaluating the local gradient in
Equation~\eqref{eq:sinkhorn_output_reverse}.

\noindent \textbf{Backward with Zero Sinkhorn Iteration.}
We first consider $T=0$ as the base case for analyzing the fused output operation.
When $T=0$, no dual-variable update is performed.
Therefore, the fused backward only needs to reverse the fused output operation.
In this case, the log-assignment and resulting assignment at each position are:
\begin{equation}
\begin{aligned}
    Z_{ij}^{l-1,l}
    &=
    \overline{S}_{\mathrm{init},ij}^{l-1,l}
    =
    \frac{
        S_{\mathrm{init},ij}^{l-1,l}
    }{\tau},\\
    S_{\mathrm{ot},ij}^{l-1,l}
    &=
    \exp\left(
        Z_{ij}^{l-1,l}
    \right).
\end{aligned}
\label{eq:sinkhorn_t0_reconstruction}
\end{equation}

Therefore, propagating the accumulated gradient in
Equation~\eqref{eq:sinkhorn_output_reverse}
to the initial assignment gives:
\begin{equation}
\begin{aligned}
    \frac{\partial\mathcal{L}}
    {\partial S_{\mathrm{init},ij}^{l-1,l}}
    &=
    \frac{\partial\mathcal{L}}
    {\partial Z_{ij}^{l-1,l}}
    \frac{\partial Z_{ij}^{l-1,l}}
    {\partial S_{\mathrm{init},ij}^{l-1,l}}\\
    &=
    \frac{1}{\tau}
    \left(
        N^l
        \frac{\partial\mathcal{L}}
        {\partial S_{\mathrm{enc},ij}^{l-1,l}}
        +
        N^{l-1}
        \frac{\partial\mathcal{L}}
        {\partial S_{\mathrm{dec},ij}^{l-1,l}}
    \right)
    S_{\mathrm{ot},ij}^{l-1,l}.
\end{aligned}
\label{eq:sinkhorn_t0_compact}
\end{equation}

\noindent \textbf{Backward with One Sinkhorn Iteration.}
When $T=1$, the forward pass performs one row dual update followed by one column dual update, and constructs the final log-assignment as:
\begin{equation}
    Z_{ij}^{l-1,l}
    =
    \overline{S}_{\mathrm{init},ij}^{l-1,l}
    +
    u_i^1
    +
    v_j^1.
\label{eq:sinkhorn_t1_log_assignment}
\end{equation}
The fused backward first accumulates
$\partial\mathcal{L}/\partial Z_{ij}^{l-1,l}$
from the two projection outputs according to
Equation~\eqref{eq:sinkhorn_output_reverse},
and then reverses the column and row dual updates.
The normalized weights associated with the row and column log-sum-exp reductions are:
\begin{equation}
\begin{aligned}
    h_{ij}^0
    &=
    \frac{
        \exp\left(
            \overline{S}_{\mathrm{init},ij}^{l-1,l}
            +
            v_j^0
        \right)
    }{
        \displaystyle
        \sum_{k=1}^{N^l}
        \exp\left(
            \overline{S}_{\mathrm{init},ik}^{l-1,l}
            +
            v_k^0
        \right)
    },\\
    q_{ij}^1
    &=
    \frac{
        \exp\left(
            \overline{S}_{\mathrm{init},ij}^{l-1,l}
            +
            u_i^1
        \right)
    }{
        \displaystyle
        \sum_{k=1}^{N^{l-1}}
        \exp\left(
            \overline{S}_{\mathrm{init},kj}^{l-1,l}
            +
            u_k^1
        \right)
    }.
\end{aligned}
\label{eq:sinkhorn_t1_normalized_weights}
\end{equation}

Directly incorporating local derivatives into the reverse-mode accumulation gives the complete gradients of the scaled initial assignment:
\begin{equation}
    \frac{\partial\mathcal{L}}
    {\partial\overline{S}_{\mathrm{init},ij}^{l-1,l}}
    =
    \underbrace{
        \frac{\partial\mathcal{L}}
        {\partial Z_{ij}^{l-1,l}}
    }_{\text{final log-assignment}}
    -
    \underbrace{
        q_{ij}^1
        \frac{\partial\mathcal{L}}
        {\partial v_j^1}
    }_{\text{column dual update}}
    -
    \underbrace{
        h_{ij}^0
        \frac{\partial\mathcal{L}}
        {\partial u_i^1}
    }_{\text{row dual update}}.
\label{eq:sinkhorn_t1_reverse_accumulation}
\end{equation}
Finally, propagating through the temperature scaling gives the gradient with respect to the initial assignment matrix:
\begin{equation}
\begin{aligned}
    \frac{\partial\mathcal{L}}
    {\partial S_{\mathrm{init},ij}^{l-1,l}}
    &=
    \frac{1}{\tau}
    \frac{\partial\mathcal{L}}
    {\partial\overline{S}_{\mathrm{init},ij}^{l-1,l}}\\
    &=
    \frac{1}{\tau}
    \left(
        \frac{\partial\mathcal{L}}
        {\partial Z_{ij}^{l-1,l}}
        -
        q_{ij}^1
        \frac{\partial\mathcal{L}}
        {\partial v_j^1}
        -
        h_{ij}^0
        \frac{\partial\mathcal{L}}
        {\partial u_i^1}
    \right).
\end{aligned}
\label{eq:sinkhorn_t1_reverse}
\end{equation}
Since the dual variables $\mathbf{u}^1$ and $\mathbf{v}^1$ are cached during the fused forward, the backward kernels can explicitly accumulate the corresponding gradients in reverse-mode order.
The normalized weights $\mathbf{h}^0$ and $\mathbf{q}^1$ are reconstructed on demand from the initial assignment and cached dual variables, without being stored as complete assignment-sized tensors.

\begin{table*}[t]
\centering
\small
\setlength{\tabcolsep}{5pt}
\resizebox{\textwidth}{!}{
\begin{tabular}{cc|ccc|ccc}
\toprule
\multirow{2}{*}{$N^{l-1}$}
& \multirow{2}{*}{$N^l$}
& \multicolumn{3}{c|}{Standard Log-domain Sinkhorn Iteration}
& \multicolumn{3}{c}{CoTAP Iteration} \\
\cmidrule(lr){3-5}\cmidrule(lr){6-8}
& & Forward (ms) $\downarrow$
& Forward+Backward (ms) $\downarrow$
& Peak Memory (GB) $\downarrow$
& Forward (ms) $\downarrow$
& Forward+Backward (ms) $\downarrow$
& Peak Memory (GB) $\downarrow$ \\
\midrule
4,096  & 512   & 1.34   & 3.87   & 0.50 & 0.74  & 3.57   & 0.22 \\
32,768 & 1,024 & 44.16  & 110.76 & 8.00 & 12.05 & 71.19  & 3.50 \\
65,536 & 2,048 & 176.24 & OOM    & OOM  & 49.47 & 466.42 & 14.00 \\
\bottomrule
\end{tabular}
}
\caption{
Performance comparison between standard Sinkhorn iterations and CoTAP Iteration.
Peak Memory denotes the peak incremental GPU memory during forward-and-backward computation.
OOM indicates that the computation exceeds the available memory of a 24-GB GPU.
}
\label{tab:cotap_iteration_efficiency}
\end{table*}

\noindent \textbf{Backward with an Arbitrary Number of Sinkhorn Iterations.}
For an arbitrary number of iterations $T$, the fused backward starts from the final log-assignment and reverses the column and row dual updates of each iteration in the opposite order of the forward computation.
For the $t$-th Sinkhorn iteration, the normalized weights associated with the row and column log-sum-exp reductions are:
\begin{equation}
\begin{aligned}
    h_{ij}^{t}
    &=
    \frac{
        \exp\left(
            \overline{S}_{\mathrm{init},ij}^{l-1,l}
            +
            v_j^{t}
        \right)
    }{
        \displaystyle
        \sum_{k=1}^{N^l}
        \exp\left(
            \overline{S}_{\mathrm{init},ik}^{l-1,l}
            +
            v_k^{t}
        \right)
    },\\
    q_{ij}^{t+1}
    &=
    \frac{
        \exp\left(
            \overline{S}_{\mathrm{init},ij}^{l-1,l}
            +
            u_i^{t+1}
        \right)
    }{
        \displaystyle
        \sum_{k=1}^{N^{l-1}}
        \exp\left(
            \overline{S}_{\mathrm{init},kj}^{l-1,l}
            +
            u_k^{t+1}
        \right)
    },
\end{aligned}
\label{eq:sinkhorn_t_normalized_weights}
\end{equation}
where $t\in\{0,\ldots,T-1\}$.
Furthermore, the local derivatives of the row and column dual updates at iteration $t$ are:
\begin{equation}
\begin{aligned}
    \frac{\partial u_i^{t+1}}
    {\partial\overline{S}_{\mathrm{init},ij}^{l-1,l}}
    &=
    \frac{\partial u_i^{t+1}}
    {\partial v_j^t}
    =
    -h_{ij}^t,\\
    \frac{\partial v_j^{t+1}}
    {\partial\overline{S}_{\mathrm{init},ij}^{l-1,l}}
    &=
    \frac{\partial v_j^{t+1}}
    {\partial u_i^{t+1}}
    =
    -q_{ij}^{t+1}.
\end{aligned}
\label{eq:sinkhorn_t_local_derivatives}
\end{equation}
The complete reverse-mode gradient of the scaled initial assignment is:
\begin{equation}
\begin{aligned}
    \frac{\partial\mathcal{L}}
    {\partial\overline{S}_{\mathrm{init},ij}^{l-1,l}}
    &=
    \underbrace{
        \frac{\partial\mathcal{L}}
        {\partial Z_{ij}^{l-1,l}}
    }_{\text{final log-assignment}}
    -
    \underbrace{
        \sum_{t=0}^{T-1}
        q_{ij}^{t+1}
        \frac{\partial\mathcal{L}}
        {\partial v_j^{t+1}}
    }_{\text{column dual updates}}
    -
    \underbrace{
        \sum_{t=0}^{T-1}
        h_{ij}^{t}
        \frac{\partial\mathcal{L}}
        {\partial u_i^{t+1}}
    }_{\text{row dual updates}}.
\end{aligned}
\label{eq:sinkhorn_t_reverse_accumulation}
\end{equation}
Propagating through the temperature scaling gives the gradient with respect to the initial assignment matrix:
\begin{equation}
\begin{aligned}
    \frac{\partial\mathcal{L}}
    {\partial S_{\mathrm{init},ij}^{l-1,l}}
    &=
    \frac{1}{\tau}
    \frac{\partial\mathcal{L}}
    {\partial\overline{S}_{\mathrm{init},ij}^{l-1,l}}\\
    &=
    \frac{1}{\tau}
    \Bigg[
        \frac{\partial\mathcal{L}}
        {\partial Z_{ij}^{l-1,l}}
        -
        \sum_{t=0}^{T-1}
        \left(
            q_{ij}^{t+1}
            \frac{\partial\mathcal{L}}
            {\partial v_j^{t+1}}
            +
            h_{ij}^{t}
            \frac{\partial\mathcal{L}}
            {\partial u_i^{t+1}}
        \right)
    \Bigg].
\end{aligned}
\label{eq:sinkhorn_t_gradient}
\end{equation}
Similar to Equation~\eqref{eq:sinkhorn_t1_reverse}, all gradients with respect to the dual variables in Equation~\eqref{eq:sinkhorn_t_reverse_accumulation}, including $\partial\mathcal{L}/\partial v_j^{t+1}$ and $\partial\mathcal{L}/\partial u_i^{t+1}$, are explicitly accumulated by the fused backward kernels in reverse-mode order, while $\textbf{h}^t$ and $\textbf{q}^t$ are reconstructed on demand from the scaled initial assignment and cached dual variables.
When $T=0$, the summations are empty and
Equation~\eqref{eq:sinkhorn_t_gradient}
reduces to Equation~\eqref{eq:sinkhorn_t0_compact}.
When $T=1$, Equation~\eqref{eq:sinkhorn_t_gradient} reduces to
Equation~\eqref{eq:sinkhorn_t1_reverse}.

\subsection{Computational Efficiency of CoTAP Iteration}

To evaluate the computational efficiency of CoTAP Iteration, we compare it with the standard log-domain Sinkhorn iteration implementation based on PyTorch.
All experiments use FP32, a batch size of 4, and four Sinkhorn iterations on an NVIDIA GeForce RTX 4090.
We report the runtime over three runs after one warm-up run and the peak GPU memory during forward-and-backward computation.

As shown in Table~\ref{tab:cotap_iteration_efficiency}, CoTAP Iteration substantially reduces the forward runtime.
When $(N^{l-1},N^l)=(32{,}768,1{,}024)$, it reduces the forward time from 44.16 ms to 12.05 ms and the forward-and-backward time from 110.76 ms to 71.19 ms, corresponding to reductions of 72.7\% and 35.7\%, respectively.
CoTAP Iteration also substantially reduces training memory.
For $(N^{l-1},N^l)=(4{,}096,5{,}12)$ and $(N^{l-1},N^l)=(32{,}768,1{,}024)$, CoTAP Iteration decreases the peak incremental memory from 0.50 GB and 8.00 GB to 0.22 GB and 3.50 GB, respectively, yielding reductions of approximately 56\%.
At the largest scale of $(65{,}536,2{,}048)$, the standard implementation runs out of memory during the backward pass, whereas CoTAP Iteration completes forward-and-backward computation with 14.00 GB of peak incremental memory.

These results show that the fused forward and explicit fused backward effectively reduce the materialization and computational-graph storage of assignment-sized intermediate tensors, enabling CoTAP to scale more efficiently to large cross-space assignments.

\section{Appendix F: Implementation Details}

\subsection{Benchmarks}

We mainly adopt benchmark configurations and data splits as in previous works~\cite{transolver, linearno}.

\noindent \textbf{Airfoil.}
The Airfoil benchmark~\cite{geo-fno} evaluates Mach number prediction around different airfoil geometries.
Geometries are generated by deforming the baseline NACA-0012 airfoil and discretized into structured meshes of size $221\times51$.
The input and output are the two-dimensional grid coordinates ($221\times51\times2$) and Mach numbers ($221\times51\times1$), respectively.
The dataset contains 1,000 training samples and 200 testing samples.

\noindent \textbf{Darcy.}
The Darcy benchmark~\cite{fno} evaluates steady-state pressure-field prediction under different porous-medium distributions, with the original physical process discretized on a $421\times421$ standard grid.
The input consists of the two-dimensional coordinates ($421\times421\times2$) and corresponding porous-medium coefficient ($421\times421\times1$) at each grid point.
The output is the pressure at each grid point, with a shape of $421\times421\times1$.
For the multiple-resolution evaluation, we construct Darcy benchmark variants at different downsampled resolutions, including $85\times85$, $141\times141$, and $211\times211$.
The dataset contains 1,000 training samples and 200 testing samples.

\noindent \textbf{Elasticity.}
The Elasticity benchmark~\cite{geo-fno} evaluates the prediction of internal stress distributions for different material structures.
Each structure is represented by 972 point-cloud observations.
The input consists of the two-dimensional coordinates of these points, with a shape of $972\times2$
The output is the stress value at each point, with a shape of $972\times1$.
The dataset contains 1,000 training samples and 200 testing samples.

\noindent \textbf{Navier-Stokes (NS2D).}
The Navier-Stokes benchmark~\cite{fno}, denoted as NS2D, evaluates autoregressive prediction of two-dimensional time-dependent flow fields.
The flow field is discretized on a $64\times64$ Cartesian grid.
The input consists of the velocity fields from the previous 10 time steps, with a shape of $64\times64\times10\times1$.
The output contains the velocity fields for the subsequent 10 time steps with a shape of $64\times64\times10\times1$.
The dataset contains 1,000 training samples and 200 testing samples.

\noindent \textbf{Pipe.}
The Pipe benchmark~\cite{geo-fno} evaluates velocity-field prediction.
The geometries are generated by varying the pipe centerline and are discretized into structured meshes of size $129\times129$.
The input consists of the two-dimensional coordinates of the mesh points, with a shape of $129\times129\times2$
The output is the velocity at each point, with a shape of $129\times129\times1$.
The dataset contains 1,000 training samples and 200 testing samples.

\noindent \textbf{Plasticity.}
The Plasticity benchmark~\cite{geo-fno} evaluates the future deformation of a plastic material.
For each sample, the die geometry is discretized into a structured mesh of size $101\times31$.
The input includes the two-dimensional coordinates ($101\times31\times2$) and corresponding external force at each mesh point ($101\times31\times1$).
The output describes the deformation of each point along four directions over the subsequent 20 time steps, with a shape of $101\times31\times20\times4$.
Unlike the autoregressive prediction used for NS2D, Plasticity encodes each target time step as a time embedding and concatenates it into the model input, allowing the deformation at each time step to be predicted independently.
The dataset contains 900 training samples with different die shapes and 80 testing samples.

\noindent \textbf{AirfRANS.}
AirfRANS~\cite{airfrans} describes two-dimensional incompressible steady-state Reynolds-Averaged Navier-Stokes flows over different airfoil geometries, Reynolds numbers, and angles of attack.
Each case is represented by a point cloud containing a varying number of observation points.
We denote the number of observation points in the $i$-th case by $N_{\mathrm{AirfRANS}}^i$.

The input consists of two-dimensional spatial coordinates ($N_{\mathrm{AirfRANS}}^i\times2$), two-dimensional inlet velocity ($N_{\mathrm{AirfRANS}}^i\times2$), Euclidean distance to the airfoil ($N_{\mathrm{AirfRANS}}^i\times1$), and two-dimensional outward-pointing unit surface normal ($N_{\mathrm{AirfRANS}}^i\times2$).
The output consists of the two-dimensional velocity ($N_{\mathrm{AirfRANS}}^i\times2$), pressure ($N_{\mathrm{AirfRANS}}^i\times1$), and turbulent kinematic viscosity ($N_{\mathrm{AirfRANS}}^i\times1$).
AirfRANS contains 1,000 samples in total.
Following previous works~\cite{transolver,linearno}, we use different data splits according to the evaluation setting.
For the main experiment, 720, 80, and 200 samples are used for training, validation, and testing, respectively.
The detailed configurations of the OOD experiments are provided in Appendix H.

\begin{table*}[!t]
    \centering
    \small
    \renewcommand{\tabcolsep}{5pt}
    \renewcommand{\arraystretch}{1.05}
    \begin{tabular}{lccccccc}
        \toprule
        Benchmark & Loss & Epochs & Warmup Epochs & LR & Optimizer & Scheduler & Batch Size \\
        \midrule
        Airfoil    & rL2 & 500 & 100 & $8\times10^{-4}$ & AdamW & OneCycle & 4 \\
        Elasticity & rL2 & 500 & 100 & $8\times10^{-4}$ & AdamW & OneCycle & 4 \\
        NS2D       & rL2 & 500 & 100 & $8\times10^{-4}$ & AdamW & OneCycle & 2 \\
        Pipe       & rL2 & 500 & 100 & $8\times10^{-4}$ & AdamW & OneCycle & 4 \\
        Plasticity & rL2 & 500 & 100 & $8\times10^{-4}$ & AdamW & OneCycle & 8 \\
        \bottomrule
    \end{tabular}
    \caption{Training hyper-parameters on the standard benchmarks. MoNo and MoNo-light use the same settings.}
    \label{tab:standard_training_settings}
\end{table*}

\subsection{Evaluation Metrics}

We use relative L2 error to evaluate prediction accuracy across all benchmarks.
For AirfRANS, we additionally evaluate the corresponding aerodynamic force coefficients and their Spearman rank correlations as previous works~\cite{transolver, linearno}.

\noindent \textbf{Relative L2 Error.}
For each test sample, the relative L2 error, termed $\mathcal{E}_{\mathrm{rL2}}$, is defined as the L2 error between the predicted and ground-truth physical fields normalized by the L2 norm of the ground truth:
\begin{equation}
\mathcal{E}_{\mathrm{rL2}}
=
\frac{
\left\|
Y-\widehat{Y}
\right\|_2
}{
\left\|
Y
\right\|_2
},
\label{eq:relative_l2_metric}
\end{equation}
where $Y$ and $\widehat{Y}$ denote the ground-truth and predicted physical fields, respectively, and $\|\cdot\|_2$ denotes the L2 norm.
The reported result is averaged over all testing samples.

\noindent \textbf{Lift Coefficient.}
The lift coefficient $C_L$ is used to evaluate global aerodynamic performance on AirfRANS.
It is computed by integrating the pressure and viscous forces over the airfoil surface to obtain the lift force, which is subsequently normalized by the freestream dynamic pressure and reference area.
The detailed computation follows Transolver~\cite{transolver} and LinearNO~\cite{linearno}.

\noindent \textbf{Spearman Rank Correlation.}
The Spearman rank correlation $\rho_L$ measures the ranking consistency between the predicted and ground-truth lift coefficients.
Given $n$ testing samples, it is defined as the Pearson correlation between the ranks of the ground-truth and predicted lift coefficient sequences:
\begin{equation}
\rho_L
=
\frac{
\operatorname{Cov}
\left(
\operatorname{rank}(\mathbf{c}_L),
\operatorname{rank}(\widehat{\mathbf{c}}_L)
\right)
}{
\sigma_{\operatorname{rank}(\mathbf{c}_L)}
\sigma_{\operatorname{rank}(\widehat{\mathbf{c}}_L)}
},
\label{eq:spearman_correlation}
\end{equation}
where $\mathbf{c}_L$ and $\widehat{\mathbf{c}}_L$ denote the ground-truth and predicted lift coefficient sequences over all testing samples, respectively.
$\operatorname{Cov}(\cdot,\cdot)$ denotes the covariance, $\operatorname{rank}(\cdot)$ denotes the ranking operation, and $\sigma_{\operatorname{rank}(\cdot)}$ denotes the standard deviation.
A value of $\rho_L$ closer to 1 indicates stronger ranking consistency.

\subsection{Baselines}

For baselines, we primarily use the results reported by LinearNO~\cite{linearno}.
For the multiple-resolution evaluation, we implement and evaluate the baselines using their publicly available code.

\subsection{MoNo}

\noindent \textbf{CoTAP.}
For all benchmarks, the temperature coefficient $\tau$ and the number of iterations $T$ in CoTAP are set to 1.0 and 8, respectively.
Each assignment projector $\mathcal{P}^l(\cdot)$ is implemented as an MLP with three hidden linear layers.
Except for Darcy, we further apply row-wise weight normalization to the final linear layer of $\mathcal{P}^1(\cdot)$.
Specifically, each row of its weight matrix is first normalized by its L2 norm and then multiplied by a learnable scale initialized to 3.0.
This design prevents the model from increasing the assignment scores of specific latent tokens merely by enlarging the norms of the corresponding output weights during training, thereby improving the stability of initial assignment learning.

\noindent \textbf{Training.}
All experiments are conducted with a random seed of 0.
MoNo and MoNo-light use the same training hyper-parameters on the standard benchmarks, as summarized in Table~\ref{tab:standard_training_settings}.
Compared with LinearNO and Transolver, MoNo and MoNo-light use the same or fewer total optimization steps for each benchmark, ensuring a fair comparison.

For the multiple-resolution evaluation on Darcy, MoNo and MoNo-light are trained for 500 epochs using the AdamW optimizer, relative L2 loss, a OneCycle scheduler with 100 warmup epochs, and a batch size of 4.
For MoNo, the learning rate is set to $8\times10^{-4}$ at all resolutions.
We observe that MoNo-light does not fully converge with a learning rate of $8\times10^{-4}$.
Therefore, we use slightly larger learning rates while keeping all other training hyper-parameters unchanged.
Specifically, the learning rate is set to $1\times10^{-3}$ at resolutions of $85\times85$ and $141\times141$, and to $1.4\times10^{-3}$ at a resolution of $211\times211$.

Following previous works~\cite{transolver,linearno}, we sample 32,000 points from each AirfRANS case during training and repeatedly perform sampling and prediction during inference until all points are covered.
Since MoNo relies on stable mapping, we find that the fully random sampling strategy with a batch size of 1 adopted in previous works causes severe training instability.
To address this issue, each sampled case contains all surface points together with randomly sampled volume points, yielding 32,000 points in total.
We further set the batch size to 4 and increase the number of training epochs to preserve the same total number of optimization steps.
Both MoNo and MoNo-light use the relative L2 loss, AdamW optimizer, and OneCycle scheduler, with the learning rate and number of warmup epochs set to $2\times10^{-4}$ and 100, respectively.
Following previous works, we jointly supervise the prediction of physical fields in the surrounding region and on the surface, assigning both loss terms a weight of 1.0.
Both OOD Reynolds and OOD Angles experiments use the same training hyper-parameters as the main AirfRANS experiment.

\noindent \textbf{Hardware and Software Environment.}
All experiments are conducted using a single NVIDIA GeForce RTX 4090 GPU with up to 600 GB of system memory.
The core software environment consists of PyTorch 2.8.0, CUDA 12.8, cuDNN 9.10.2, and Triton 3.4.0.

\section{Appendix G: Efficiency Analysis}

Table~\ref{tab:efficiency-analysis} reports the parameter counts and computational costs.
All results use FP32 inference with a batch size of one.
All models use two-channel inputs and single-channel outputs for evaluation.
MoNo-light consistently achieves the lowest computational cost across all input sizes, while MoNo also exhibits increasingly clear computational advantages and outperforms all baselines at larger input scales.

\section{Appendix H: Detailed OOD Results}

The training and test ranges for the OOD Reynolds and OOD Angles settings are summarized in Table~\ref{tab:ood_generalization_configuration}.
Table~\ref{tab:ood_benchmarks_full} further presents the complete OOD results on AirfRANS, including the lift coefficient error $C_L$ and its Spearman rank correlation $\rho_L$.
MoNo achieves the best performance across all metrics under both OOD settings.
On OOD Reynolds, MoNo achieves a $C_L$ error of 0.1050, outperforming the strongest baseline, Transolver, by 35.3\%, while improving the $\rho_L$ from 0.9951 (LinearNO) to 0.9983.
On OOD Angles, MoNo further achieves a $C_L$ error of 0.0782, outperforming the strongest baseline, LinearNO, by 20.8\%, and improves $\rho_L$ from 0.9963 to 0.9970.
These results demonstrate that MoNo can extrapolate not only local physical fields but also global aerodynamic performance under unseen flow conditions.

\begin{table}[!t]
\centering
\small
\setlength{\tabcolsep}{5pt}
\renewcommand{\arraystretch}{0.95}
\begin{tabular}{l|ccc}
\toprule
Model & Parameters (M) & Input Points $N^0$ & GFLOPs \\
\midrule
\multirow{7}{*}{GNOT} & \multirow{7}{*}{6.73}
& 1,024  & 6.94 \\
& & 2,048  & 13.88 \\
& & 4,096  & 27.76 \\
& & 8,192  & 55.52 \\
& & 16,384 & 111.04 \\
& & 32,768 & 222.09 \\
& & 65,536 & 444.18 \\
\midrule
\multirow{7}{*}{Transolver} & \multirow{7}{*}{2.81}
& 1,024  & 3.08 \\
& & 2,048  & 6.14 \\
& & 4,096  & 12.27 \\
& & 8,192  & 24.53 \\
& & 16,384 & 49.06 \\
& & 32,768 & 98.10 \\
& & 65,536 & 196.19 \\
\midrule
\multirow{7}{*}{Transolver++} & \multirow{7}{*}{1.74}
& 1,024  & 1.90 \\
& & 2,048  & 3.79 \\
& & 4,096  & 7.58 \\
& & 8,192  & 15.15 \\
& & 16,384 & 30.30 \\
& & 32,768 & 60.59 \\
& & 65,536 & 121.17 \\
\midrule
\multirow{7}{*}{LinearNO} & \multirow{7}{*}{1.77}
& 1,024  & 2.08 \\
& & 2,048  & 4.15 \\
& & 4,096  & 8.30 \\
& & 8,192  & 16.60 \\
& & 16,384 & 33.21 \\
& & 32,768 & 66.41 \\
& & 65,536 & 132.83 \\
\midrule
\multirow{7}{*}{\textbf{MoNo-light}} & \multirow{7}{*}{1.73}
& 1,024  & 1.23 \\
& & 2,048  & 1.58 \\
& & 4,096  & 2.29 \\
& & 8,192  & 3.70 \\
& & 16,384 & 6.52 \\
& & 32,768 & 12.17 \\
& & 65,536 & 23.46 \\
\midrule
\multirow{7}{*}{\textbf{MoNo}} & \multirow{7}{*}{6.90}
& 1,024  & 8.41 \\
& & 2,048  & 9.82 \\
& & 4,096  & 12.64 \\
& & 8,192  & 18.28 \\
& & 16,384 & 29.56 \\
& & 32,768 & 52.13 \\
& & 65,536 & 97.26 \\
\bottomrule
\end{tabular}

\caption{Efficiency comparison with different numbers of input points. MoNo achieves higher computational efficiency.}
\label{tab:efficiency-analysis}
\end{table}

\begin{table*}[!t]
    \centering
    \small
    \renewcommand{\tabcolsep}{6pt}
    \renewcommand{\arraystretch}{1.1}
    \begin{tabular}{lcc}
        \toprule
        Split & OOD Reynolds: Reynolds Number Range & OOD Angles: AoA Range \\
        \midrule
        Training Set
        & $[3\times10^{6},\,5\times10^{6}]$
        & $[-2.5^{\circ},\,12.5^{\circ}]$ \\
        Test Set
        & $[2\times10^{6},\,3\times10^{6}]
          \cup[5\times10^{6},\,6\times10^{6}]$
        & $[-5^{\circ},\,-2.5^{\circ}]
          \cup[12.5^{\circ},\,15^{\circ}]$ \\
        \bottomrule
    \end{tabular}
    \caption{Training and test ranges for the AirfRANS OOD Reynolds and OOD Angles evaluations.}
    \label{tab:ood_generalization_configuration}
\end{table*}

\begin{table*}[!t]
  \centering
  \small
  \renewcommand\tabcolsep{6pt}
  \renewcommand{\arraystretch}{1.05}
  \begin{tabular}{l|cccccccc}
    \toprule
    
    \multirow{2}{*}{Method} & \multicolumn{4}{c}{AirfRANS (OOD Reynolds)} & \multicolumn{4}{c}{AirfRANS (OOD Angles)} \\
    \cmidrule(lr){2-5} \cmidrule(lr){6-9}
    & Vol. ($\downarrow$) & Surf. ($\downarrow$) & $C_L$ ($\downarrow$) & $\rho_L$ ($\uparrow$) & Vol. ($\downarrow$) & Surf. ($\downarrow$) & $C_L$ ($\downarrow$) & $\rho_L$ ($\uparrow$) \\

    \midrule
    MLP & 0.0669 & 0.1153 & 0.6205 & 0.9578 & 0.1309 & 0.3311 & 0.4128 & 0.9572 \\
    PointNet~\cite{pointnet} & 0.0838 & 0.1403 & 0.3836 & 0.9806 & 0.2021 & 0.4649 & 0.4425 & 0.9784 \\
    Graph U-Net~\cite{graphunet} & 0.0538 & 0.1168 & 0.4664 & 0.9645 & 0.0979 & 0.2391 & 0.3756 & 0.9816 \\
    MeshGraphNet~\cite{meshgraphnet} & 0.2789 & 0.2382 & 1.7718 & 0.7631 & 0.4902 & 1.1071 & 0.6525 & 0.8927 \\
    
    \midrule
    GNO~\cite{gno} & 0.0833 & 0.1562 & 0.4408 & 0.9878 & 0.1626 & 0.2359 & 0.3038 & 0.9884 \\
    GALERKIN~\cite{galerkin} & 0.0330 & 0.0972 & 0.4615 & 0.9826 & 0.0577 & 0.2773 & 0.3814 & 0.9821 \\
    GNOT~\cite{gnot} & 0.0305 & 0.0959 & 0.3268 & 0.9865 & 0.0471 & 0.3466 & 0.3497 & 0.9868 \\
    GINO~\cite{gino} & 0.0839 & 0.1825 & 0.4180 & 0.9645 & 0.1589 & 0.2469 & 0.2583 & 0.9923 \\
    LNO~\cite{lno} & 0.0825 & 0.1762 & 0.7440 & 0.9399 & 0.0346 & 0.0790 & 0.3732 & 0.9905 \\
    Transolver~\cite{transolver} & 0.0122 & 0.0550 & \underline{0.1622} & 0.9904 & 0.0480 & 0.2335 & 0.2438 & 0.9948 \\
    LinearNO~\cite{linearno} & 0.0112 & 0.0372 & 0.2400 & \underline{0.9951} & 0.0464 & 0.2500 & 0.0987 & \underline{0.9963} \\

    \midrule
    \textbf{MoNo-Light (ours)} & \underline{0.0090} & \underline{0.0146} & 0.1816 & 0.9936 & \underline{0.0222} & \underline{0.0553} & \underline{0.0951} & 0.9953 \\
    \textbf{MoNo (ours)} & \textbf{0.0066} & \textbf{0.0121} & \textbf{0.1050} & \textbf{0.9983} & \textbf{0.0097} & \textbf{0.0248} & \textbf{0.0782} & \textbf{0.9970} \\

    \bottomrule
  \end{tabular}
  \caption{Comparison on the AirfRANS OOD Reynolds and OOD Angles benchmarks. We evaluate both the surrounding (Vol.) and surface (Surf.) physical fields, as well as the lift coefficient ($C_L$). $\rho_L$ denotes Spearman's rank correlation for $C_L$.}
  \label{tab:ood_benchmarks_full}
\end{table*}

\section{Appendix I: More Ablation Studies}

\begin{table}[t]
  \centering
  \small
  \renewcommand\tabcolsep{5pt}
  \renewcommand{\arraystretch}{1.05}
  \begin{tabular}{ccc|ccc}
    \toprule
    $N^1$ & $D_g$ \& $D_f$ & Params. (M)
    & Airfoil & Elasticity & Pipe \\
    \midrule
    256  & 48  & 0.44  & 0.0060 & 0.0127 & 0.0037 \\
    512  & 48  & 0.46  & 0.0053 & 0.0123 & 0.0036 \\
    512  & 96  & 1.73  & 0.0048 & 0.0042 & 0.0027 \\
    512  & 192 & 6.71  & 0.0048 & 0.0033 & 0.0022 \\
    1024 & 192 & 6.90  & 0.0048 & 0.0033 & 0.0021 \\
    1024 & 384 & 26.77 & 0.0046 & 0.0035 & 0.0021 \\
    2048 & 384 & 27.50 & 0.0049 & 0.0031 & 0.0023 \\
    \bottomrule
  \end{tabular}
  \caption{Ablation studies of the number of tokens in the first latent space and token dimensions in relative L2 error ($\downarrow$).}
  \label{tab:ablation_multi-hyperparameter}
\end{table}

\noindent \textbf{Latent Token Number and Token Dimensions.}
Table~\ref{tab:ablation_multi-hyperparameter} analyzes the effects of the number of latent tokens $N^1$ in the first latent space $\Omega^1$ and the token dimensions of anchor tokens and physical state tokens.
Increasing $N^1$ and the token dimensions initially improves performance across the three tasks, after which performance gradually saturates, and further scaling provides no consistent gains.

Compared with the initial configuration using $N^1=256$ and token dimensions of 48, MoNo-light uses $N^1=512$ and token dimensions of 96, reducing the errors on Airfoil, Elasticity, and Pipe by 20.0\%, 66.9\%, and 27.0\%, respectively.
MoNo further uses $N^1=1024$ and token dimensions of 192, reducing the errors on the three tasks by 20.0\%, 74.0\%, and 43.2\%, respectively, compared with the initial configuration.
These results show that MoNo-light achieves a favorable balance between model scale and overall solution accuracy, while MoNo further improves overall accuracy through increased representational capacity.

\noindent \textbf{CoTAP Hyper-parameters.}
Table~\ref{tab:ablation_cotap_hyperparameters} analyzes the effects of the temperature coefficient $\tau$ and the number of iterations $T$ with MoNo-light.
The temperature coefficient $\tau$ controls the sharpness of the assignment distribution.
A smaller $\tau$ weakens entropy regularization and produces sharper assignments, whereas a larger $\tau$ yields smoother assignments.
With $T=8$, both $\tau=0.5$ and $\tau=1.0$ achieve stable performance, whereas increasing $\tau$ to 4.0 reduces prediction performance.
These results indicate that latent tokens require sufficiently sharp assignments to better represent local physical information.
Increasing the number of iterations allows the assignment matrix to satisfy the marginal constraints more closely, effectively imposing stronger normalization, but also makes convergence more difficult.
The results show that increasing $T$ from 8 to 12 provides no consistent performance improvement and noticeably reduces prediction accuracy on Elasticity.

In all experiments, we use $\tau=1.0$ and $T=8$ as the default configuration to balance assignment sharpness, marginal constraint enforcement, and training convergence.

\begin{table}[!t]
    \centering
    \small
    \renewcommand{\tabcolsep}{5pt}
    \renewcommand{\arraystretch}{1.05}
    \begin{tabular}{cc|ccc}
        \toprule
        $\tau$ & $T$ & Airfoil & Elasticity & Pipe \\
        \midrule
        0.5 & 8  & 0.0048 & 0.0041 & 0.0023 \\
        4.0 & 8  & 0.0050 & 0.0046 & 0.0027 \\
        1.0 & 4  & 0.0045 & 0.0043 & 0.0026 \\
        1.0 & 12 & 0.0048 & 0.0054 & 0.0024 \\
        \midrule
        1.0 & 8  & 0.0048 & 0.0042 & 0.0027 \\
        \bottomrule
    \end{tabular}
    \caption{Ablation studies of $\tau$ and $T$ in CoTAP by relative L2 error ($\downarrow$).}
    \label{tab:ablation_cotap_hyperparameters}
\end{table}

\begin{table}[!t]
    \centering
    \small
    \renewcommand{\tabcolsep}{3pt}
    \renewcommand{\arraystretch}{1.05}
    \begin{tabular}{cc|ccc}
        \toprule
        Input for $\mathbf{S}_{\mathrm{init}}^{l-1,l}$
        & Reduction Factor & Airfoil & Elasticity & Pipe \\
        \midrule
        $F_{\mathrm{enc}}^{l-1}$ & $\times2$ & 0.0050 & 0.0058 & 0.0031 \\
        $G^{l-1}$                & $\times1$ & 0.0046 & 0.0050 & 0.0026 \\
        $G^{l-1}$                & $\times4$ & 0.0048 & 0.0064 & 0.0030 \\
        \midrule
        $G^{l-1}$                & $\times2$ & 0.0048 & 0.0042 & 0.0027 \\
        \bottomrule
    \end{tabular}
    \caption{Ablation studies of the input for initial assignment construction and the token reduction factor between adjacent latent spaces in relative L2 error ($\downarrow$).}
    \label{tab:ablation_cotap_designs}
\end{table}

\noindent \textbf{Architectural Designs.}
Table~\ref{tab:ablation_cotap_designs} analyzes the input for initial assignment construction and the token reduction factor between adjacent latent spaces.
Constructing assignments from $F_{\mathrm{enc}}^{l-1}$ enables case-specific projections conditioned on the current physical states but substantially reduces the stability of CoTAP across latent spaces.
With limited training data, this additional flexibility encourages the model to fit complex case-specific cross-space projections instead of sufficiently learning physical interactions within each latent space.

For the token reduction factor, $\times1$ preserves the number of tokens and only redistributes them between adjacent spaces, thereby providing no effective hierarchical information aggregation.
In contrast, $\times4$ compresses physical states too rapidly and causes substantial information loss.
Therefore, we construct CoTAP from $G^{l-1}$ and use a token reduction factor of $\times2$ by default to balance cross-space projection stability, progressive information aggregation, and representational capacity at each scale.

\section{Appendix J: The Use of Large Language Models}

All major technical contributions and theoretical derivations in this work were conceived and developed by the authors.
Large language models (LLMs) were used for language polishing, auxiliary analysis, and checking theoretical derivations.
All LLM-assisted descriptions and analyses were carefully reviewed and verified by the authors.

\bibliography{aaai2027}

\end{document}